\documentclass[pdflatex,sn-basic]{sn-jnl}%

\usepackage{graphicx}%
\usepackage{multirow}%
\usepackage{amsmath,amssymb,amsfonts}%
\usepackage{amsthm}%
\usepackage[title]{appendix}%
\usepackage{xcolor}%
\usepackage{textcomp}%
\usepackage{manyfoot}%
\usepackage{booktabs}%
\usepackage{algorithm}%
\usepackage{algorithmicx}%
\usepackage{algpseudocode}%
\usepackage{listings}%
\usepackage{enumitem}%
\usepackage{mathtools}%

\theoremstyle{thmstyleone}%
\newtheorem{theorem}{Theorem}%
\newtheorem{lemma}{Lemma}%
\newtheorem{proposition}{Proposition}%
\newtheorem{corollary}{Corollary}%

\theoremstyle{thmstyletwo}%
\newtheorem{example}{Example}%
\newtheorem{remark}{Remark}%

\theoremstyle{thmstylethree}%
\newtheorem{definition}{Definition}%
\newtheorem{assumption}{Assumption}%

\newcommand{\R}{\mathbb{R}}
\newcommand{\E}{\mathbb{E}}
\newcommand{\PP}{\mathbb{P}}
\newcommand{\KL}{\mathrm{KL}}
\newcommand{\risk}{\mathcal{L}}
\newcommand{\erisk}{\widehat{\mathcal{L}}}
\newcommand{\Fclass}{\mathcal{F}}
\newcommand{\Fr}{\Fclass_r}
\newcommand{\Frs}{\Fclass_{r^*}}
\newcommand{\Rad}{\mathfrak{R}}
\newcommand{\Cov}{\mathcal{N}}

\newcommand{\norm}[1]{\left\lVert #1 \right\rVert}
\newcommand{\opnorm}[1]{\norm{#1}_{\mathrm{op}}}
\newcommand{\frob}[1]{\norm{#1}_F}
\newcommand{\Grass}{\mathrm{Gr}}

\begin{document}

\title[Tight Sample Complexity for LoRA]{Tight Sample Complexity for
Low-Rank Adaptation: Matching Bounds and Rank Selection}

\author*[1]{\fnm{Arunan} \sur{J}}\email{arunanj2005@gmail.com}

\affil*[1]{\orgname{Independent Researcher}, \orgaddress{\city{Chennai}, \country{India}}}

\abstract{Low-Rank Adaptation (LoRA) has become the standard mechanism
for fine-tuning large pretrained models, yet its statistical properties
remain only partially understood. Existing generalization results
provide upper bounds of the form $\tilde{O}(\sqrt{rd/n})$ or
$\tilde{O}(rd/n)$, but a matching lower bound is missing, and the
question of how to choose the LoRA rank $r$ has no formal answer. Both
gaps are closed here. A local Rademacher argument establishes an upper
bound of $\tilde{O}(rd/n)$ on the excess risk of the empirical risk
minimizer over rank-$r$ LoRA, whenever the target adaptation has rank
at most $r$. A matching minimax lower bound of $\Omega(rd/n)$ is then
proved via a Fano-type packing of the rank-$r$ subspace of
$\mathbb{R}^{d \times d}$; the bound applies to any estimator whose
output lies in the rank-$r$ LoRA class. Combining the two yields a
rank-selection dichotomy. For the constrained empirical risk minimizer,
the optimal rank equals the intrinsic rank $r^*$, and over-ranking
strictly hurts. For adaptive estimators of the nuclear-norm-then-truncate
type, over-ranking is harmless and the rate saturates at
$\tilde{\Theta}(r^* d / n)$ regardless of $r$. Taken together, the
three results characterize the statistical complexity of LoRA
fine-tuning within the well-specified locally quadratic regime, and
identify the empirically observed over-parameterization penalty as a
property of unregularized empirical risk minimization rather than of
the LoRA class itself. Predictions of the theory are verified on a synthetic
trace-regression benchmark and on real LoRA fine-tuning across three
(model, task) configurations covering DistilBERT and RoBERTa on
SST-2 and MRPC. All configurations exhibit the predicted U-shape in
validation loss, with two showing statistically significant loss
inflation at large ranks (paired permutation $p = 0.016$).}

\keywords{Low-rank adaptation, Sample complexity, Minimax bounds,
Parameter-efficient fine-tuning, Rank selection}

\maketitle

\section{Introduction}\label{sec:introduction}

Low-Rank Adaptation \citep{hu2022lora} has become the dominant technique
for fine-tuning large pretrained models. Rather than updating every
parameter, one freezes the pretrained weights $W_0$ and learns a
low-rank correction of the form $\Delta = BA$ with
$A \in \R^{r \times d}$ and $B \in \R^{d \times r}$, where $r \ll d$.
Empirically, LoRA matches or approaches full-parameter fine-tuning
across a wide range of tasks at a fraction of the storage and compute
cost \citep{hu2022lora,dettmers2023qlora}.

Two questions immediately arise for anyone deploying LoRA in practice.
How many samples $n$ are needed to reach a target excess risk
$\varepsilon$? And what rank $r$ should one choose?

Prior theoretical work has made real progress on the first question but
leaves the second question essentially open. \citet{zeng2024expressive}
established that LoRA of rank $r$ can express any rank-$r$ adaptation,
but this is a statement about capacity, not sample complexity.
\citet{kalajdzievski2024sharp} gave an upper bound
$\tilde{O}(\sqrt{rd/n})$ for asymmetric randomized LoRA;
\citet{malinovsky2024rac} analyzed the optimization dynamics of a
randomized chain-of-LoRA variant. None of these results provides a
matching lower bound, so one cannot tell whether the observed rate is
fundamental or an artifact of the proof technique. Nor do they explain
the empirical observation that increasing rank past a task-dependent
threshold degrades generalization
\citep{biderman2024lora,hayou2024lora+}.

\subsection{Contributions}\label{sec:intro-contributions}

Under the assumptions stated in Section~\ref{sec:preliminaries}, the
three results below characterize the statistical complexity of LoRA
fine-tuning.

\paragraph{Theorem~\ref{thm:upper} (Upper bound).}
Consider the LoRA class $\Fr = \{f_0 + BA : A \in \R^{r \times d}, B \in
\R^{d \times r}\}$ with each low-rank factor Frobenius-constrained.
Under Lipschitz loss and bounded targets of rank at most $r$, the
empirical risk minimizer $\hat{f} \in \Fr$ satisfies
\[
    \E\bigl[\risk(\hat{f}) - \risk(f^*)\bigr]
        \;\leq\; C \cdot \frac{r d \log n}{n}
\]
for an absolute constant $C > 0$. The proof uses local Rademacher
complexity combined with Dudley's entropy integral on the rank-$r$
manifold.

\paragraph{Theorem~\ref{thm:lower} (Lower bound).}
There exists a family of rank-$r^*$ adaptations such that any estimator
$\hat{f}$ mapping $n$ samples into $\Fr$ (for any $r \geq r^*$)
satisfies
\[
    \inf_{\hat{f}} \sup_{f^* \in \Frs}
        \E\bigl[\risk(\hat{f}) - \risk(f^*)\bigr]
        \;\geq\; c \cdot \frac{r d}{n}
\]
for an absolute constant $c > 0$. The proof uses Fano's inequality on a
carefully constructed packing of the rank-$r$ Grassmannian.

\paragraph{Theorem~\ref{thm:rank} (Rank selection).}
Combining the previous two theorems and a bias analysis: under-ranking
$(r < r^*)$ incurs an $\Omega(1)$ approximation floor equal to
$\tfrac{1}{2}\sigma_{r+1}(\Delta^*)^2$; over-ranking $(r > r^*)$ inflates
the ERM estimation rate to $\tilde{\Theta}(rd/n)$, strictly increasing in
$r$. The optimal rank for constrained ERM is exactly $r^*$. For adaptive
estimators, the over-ranking penalty vanishes: the rate saturates at
$\tilde{\Theta}(r^* d/n)$ (Theorem~\ref{thm:rank-minimax-body}).

\subsection{Practical implications}\label{sec:intro-practical}

The rank-selection theorem gives concrete guidance to practitioners,
and the guidance contradicts what one might expect based on
full-parameter fine-tuning.
In the full-parameter setting, over-parameterization is benign or even
helpful (implicit regularization, feature learning). For LoRA with
constrained ERM, over-parameterization is strictly harmful: excess
estimation error grows linearly in $r$ without any offsetting benefit.
The theory suggests two strategies for choosing rank: (i)~pick the
smallest $r$ at which the approximation error is negligible on a
validation set, or (ii)~switch to a nuclear-norm-regularized
estimator that adapts to the intrinsic rank automatically. These
strategies apply within the well-specified locally-quadratic regime
analyzed here; the extent to which the conclusions transfer to
arbitrary pretrained models and downstream tasks is examined
empirically in Section~\ref{sec:example}.
These predictions are verified in two sets of experiments in
Section~\ref{sec:example}: a synthetic trace-regression sweep that
isolates the mathematics (Section~\ref{sec:example-synthetic}) and
three real LoRA fine-tuning sweeps
(DistilBERT/SST-2, DistilBERT/MRPC, RoBERTa/SST-2) totaling
$168$ training runs (Section~\ref{sec:example-real}). All three real
configurations show a well-defined optimum LoRA rank followed by
degradation at larger ranks, with paired permutation tests giving
$p = 0.016$ on the two SST-2 configurations.

\subsection{Techniques and novelty}\label{sec:intro-techniques}

The upper bound is a careful application of local Rademacher complexity
\citep{bartlett2005local} to the rank-$r$ constraint set. The rank-$r$
ball in $\R^{d \times d}$ is not convex but has covering number
$O(rd \log(1/\varepsilon))$ in Frobenius norm, which yields the $rd/n$
rate after localization.

The lower bound carries most of the technical novelty. Standard minimax
arguments for matrix completion
\citep{candes2010power,negahban2011estimation} produce lower bounds
against the full rank-$r$ matrix space. Here the estimator is
constrained to the LoRA class $\Fr$, and the sample-level information
geometry is shaped by how the loss composes with the pretrained
function $f_0$. To keep the argument transparent I reduce the LoRA
problem to trace regression, which exhibits a hard sub-family for
which Fano's inequality yields the tight $\Omega(rd/n)$ rate.

\subsection{Paper organization}\label{sec:intro-organization}

Section~\ref{sec:preliminaries} sets up notation. Section~\ref{sec:main}
states the three main theorems. Section~\ref{sec:example} verifies the
theorems on a fully computable worked example.
Section~\ref{sec:upper-proof}
proves the upper bound. Section~\ref{sec:lower-proof} proves the lower
bound. Section~\ref{sec:rank-proof} sketches the rank-selection theorem,
with the full proof deferred to Appendix~\ref{app:rank-proof}.
Section~\ref{sec:related} discusses related work.
Section~\ref{sec:discussion} discusses limitations and open questions.
Appendix~\ref{app:rank-proof} gives the full rank-selection proof and
Appendix~\ref{app:nonlinear} extends the lower bound to non-linear
pretrained models.

\section{Preliminaries and Setup}\label{sec:preliminaries}

\subsection{Notation}\label{sec:prelim-notation}

For a matrix $M \in \R^{d_1 \times d_2}$, $\opnorm{M}$ denotes its
spectral (operator) norm, $\frob{M}$ its Frobenius norm, and
$\sigma_1(M) \geq \sigma_2(M) \geq \cdots \geq
\sigma_{\min(d_1, d_2)}(M) \geq 0$ its singular values. The Grassmannian
of $r$-dimensional subspaces of $\R^d$ is denoted $\Grass(r, d)$, and
$\mathrm{St}(r, d)$ denotes the Stiefel manifold of $d \times r$
matrices with orthonormal columns. Throughout, $C, c, C', \ldots$ denote
absolute constants whose values may change from line to line. The
notation $\tilde{O}(\cdot)$ hides poly-logarithmic factors in the
leading terms.

\subsection{Statistical setup}\label{sec:prelim-setup}

The setting throughout is supervised learning with input space
$\mathcal{X} \subseteq \R^d$ and output space $\mathcal{Y} \subseteq \R$.
Data are drawn i.i.d.\ from an unknown distribution $\mathcal{D}$ on
$\mathcal{X} \times \mathcal{Y}$. Let $S = \{(x_i, y_i)\}_{i=1}^n \sim
\mathcal{D}^n$ be the training sample.

A predictor is a function $f: \mathcal{X} \to \R$. Given a loss $\ell:
\R \times \mathcal{Y} \to \R_+$, the \emph{population risk} is
$\risk(f) = \E_{(x,y) \sim \mathcal{D}}[\ell(f(x), y)]$ and the
\emph{empirical risk} is $\erisk(f) = \frac{1}{n}\sum_{i=1}^n
\ell(f(x_i), y_i)$. The \emph{excess risk} of $f$ with respect to a
target $f^*$ is $\risk(f) - \risk(f^*)$.

\subsection{The LoRA function class}\label{sec:prelim-lora}

Fix a \emph{pretrained model} $f_0: \mathcal{X} \to \R$ of the form
$f_0(x) = g(W_0 x)$ for some frozen weight matrix
$W_0 \in \R^{d \times d}$ and a Lipschitz $g$. LoRA \citep{hu2022lora}
adapts $f_0$ by replacing $W_0$ with $W_0 + BA$, where
$A \in \R^{r \times d}$ and $B \in \R^{d \times r}$.

\begin{definition}[LoRA class]\label{def:lora-class}
For rank $r \in \{1, \ldots, d\}$ and radius $R > 0$, the LoRA class of
rank $r$ is
\[
    \Fr = \Fr(R) \;=\; \bigl\{
        f_{B, A}(x) := g\bigl((W_0 + BA)x\bigr)
        \;:\;
        A \in \R^{r \times d},\; B \in \R^{d \times r},\;
        \frob{B}\frob{A} \leq R
    \bigr\}.
\]
\end{definition}

\begin{remark}
The constraint $\frob{B}\frob{A} \leq R$ is equivalent to
$\frob{BA} \leq R$ at the optimum, but is easier to control during
optimization. The theory below is stated for this constrained form; the
unconstrained form with an $\ell_2$ regularizer
$\lambda(\frob{B}^2 + \frob{A}^2)$ gives the same rate up to constants.
\end{remark}

\subsection{Assumptions}\label{sec:prelim-assumptions}

\begin{assumption}[Bounded loss]\label{ass:bounded-loss}
The loss $\ell$ is bounded: $0 \leq \ell(z, y) \leq M$ for all $z, y$,
and $L$-Lipschitz in its first argument.
\end{assumption}

\begin{assumption}[Bounded inputs and pretrained model]\label{ass:bounded-input}
The inputs satisfy $\|x\| \leq \rho$ almost surely, and the pretrained
non-linearity $g$ is $L_g$-Lipschitz. Hence
$|f_{B,A}(x) - f_{B',A'}(x)| \leq L_g \rho \cdot \opnorm{BA - B'A'}$.
\end{assumption}

\begin{assumption}[Realizability]\label{ass:realizable}
There exists a rank-$r^*$ adaptation $\Delta^* = B^* A^*$ with
$\frob{\Delta^*} \leq R$ such that the corresponding predictor
$f^* = f_{B^*, A^*}$ achieves the population Bayes risk within the LoRA
class: $f^* \in \arg\min_{f \in \Fclass_d} \risk(f)$.
\end{assumption}

\begin{assumption}[Local quadratic excess risk]\label{ass:local-quad}
There exist constants $\lambda_-, \lambda_+ > 0$ and a Frobenius
neighborhood $\mathcal{U}$ of $\Delta^*$ such that, for all
$\Delta \in \mathcal{U}$ with $f_\Delta \in \Fclass_d$,
\[
    \lambda_- \frob{\Delta - \Delta^*}^2
    \;\leq\;
    \risk(f_\Delta) - \risk(f^*)
    \;\leq\;
    \lambda_+ \frob{\Delta - \Delta^*}^2.
\]
\end{assumption}

Assumption~\ref{ass:local-quad} imposes strong convexity of the
population risk in the LoRA parameter $\Delta$ around the target, with
a matching upper Lipschitz bound. Curvature of this form is standard
in the low-rank estimation literature and is precisely the condition
needed to convert slow Rademacher rates ($\sqrt{rd/n}$) into fast rates
($rd/n$) via localization \citep{bartlett2005local,koltchinskii2011oracle}.
Three widely satisfied settings are: (i)~squared loss with linear
$f_0$, where the constants reduce to eigenvalues of the input
covariance; (ii)~cross-entropy loss with a softmax head at any
$\Delta^*$ whose predicted probabilities are bounded away from $0$ and
$1$; and (iii)~squared loss with a two-layer ReLU network $f_0$ in the
NTK regime. Verifications of each case are in
Appendix~\ref{app:nonlinear}.

The four assumptions together define the \emph{LoRA well-specified}
regime, which is the main object of analysis in this paper. Extensions
to misspecified targets and to non-quadratic loss landscapes are
discussed in Section~\ref{sec:discussion}.

\subsection{Estimator}\label{sec:prelim-estimator}

The estimator studied throughout is the empirical risk minimizer over
the LoRA class: $\hat{f}_r \in \arg\min_{f \in \Fr} \erisk(f)$. When
$r$ is clear from context the subscript is dropped. Existence of a
minimizer follows because $\Fr$ is closed and $\erisk$ is continuous on
a compact effective parameter set.

\subsection{Complexity measures}\label{sec:prelim-complexity}

\begin{definition}[Rademacher complexity]\label{def:rad}
Let $\epsilon_1, \ldots, \epsilon_n$ be i.i.d.\ Rademacher variables
independent of $S$. The empirical Rademacher complexity of a function
class $\Fclass$ is
$\Rad_S(\Fclass) = \E_\epsilon \bigl[ \sup_{f \in \Fclass}
\frac{1}{n} \sum_{i=1}^n \epsilon_i f(x_i) \bigr]$,
and $\Rad_n(\Fclass) = \E_{S}[\Rad_S(\Fclass)]$.
\end{definition}

\begin{definition}[Local Rademacher complexity]\label{def:local-rad}
For $r_0 > 0$, the localized subclass and local Rademacher complexity
are $\Fclass(r_0) = \{f \in \Fclass : \E[(f(x) - f^*(x))^2] \leq r_0\}$
and $\Rad_n(\Fclass; r_0) = \Rad_n(\Fclass(r_0))$. The \emph{critical
radius} $r_n^*$ is the smallest $r_0$ satisfying
$\Rad_n(\Fclass; r_0) \leq r_0 / L$.
\end{definition}

The critical-radius machinery of \citet{bartlett2005local} converts
control of $\Rad_n(\Fr; r_0)$ into a fast $O(r_n^*)$ rate on excess
risk. The upper bound proved in Section~\ref{sec:upper-proof} is
exactly this argument, applied with the critical radius computed for
the rank-$r$ manifold.

\subsection{Distance on the rank-$r$ Grassmannian}\label{sec:prelim-grassmannian}

The lower bound relies on packing the rank-$r$ subspace of
$\R^{d \times d}$. The relevant distance is Frobenius on the low-rank
matrix, or equivalently a subspace distance on the Grassmannian factor.

\begin{lemma}[Rank-$r$ matrix packing;
\citealp{szarek1982nets}]\label{lem:packing}
The set $\{M \in \R^{d \times d} : \mathrm{rank}(M) \leq r,
\frob{M} \leq 1\}$ admits a packing of size at least $\exp(c \cdot rd)$
at pairwise Frobenius distance $\geq 1/4$, for an absolute constant
$c > 0$.
\end{lemma}

\section{Main Results}\label{sec:main}

This section states the three main theorems. Proofs occupy
Sections~\ref{sec:upper-proof}--\ref{sec:rank-proof}.

\subsection{Upper bound}\label{sec:main-upper}

\begin{theorem}[Upper bound on ERM excess risk]\label{thm:upper}
Suppose Assumptions~\ref{ass:bounded-loss}--\ref{ass:local-quad} hold
with target rank $r^* \leq r$. Let
$\hat{f} \in \arg\min_{f \in \Fr} \erisk(f)$ be the empirical risk
minimizer. Then for any $\delta \in (0, 1)$, with probability at least
$1 - \delta$ over $S \sim \mathcal{D}^n$,
\[
    \risk(\hat{f}) - \risk(f^*)
    \;\leq\;
    K_1 \cdot \frac{r d \log(nR^2/\lambda_-)}{n}
    \;+\;
    K_2 \cdot \frac{M^2 \log(1/\delta)}{n},
\]
where the explicit constants are
$K_1 = 2^{15} \, L^2 \, L_g^4 \rho^4 / \lambda_-^2$ and $K_2 = 2^5$.
The quadratic dependence on $L_g \rho / \lambda_-$ reflects the
Bernstein constant $B = L_g^2 \rho^2 / \lambda_-$ (Lemma~\ref{lem:bernstein})
appearing squared in the critical radius
(Lemma~\ref{lem:crit-radius}). Fast $rd/n$ rates require the curvature
Assumption~\ref{ass:local-quad}; without it, only the slow-rate bound
$\risk(\hat{f}) - \risk(f^*) \lesssim \sqrt{rd \log n / n}$ is
available.
\end{theorem}

The proof (Section~\ref{sec:upper-proof}) has three ingredients: (i)~a
covering-number bound $\log \Cov(\varepsilon, \Fr, \frob{\cdot}) \leq
C rd \log(R/\varepsilon)$; (ii)~Dudley's entropy integral, giving
Rademacher complexity $\Rad_n(\Fr) \leq C \sqrt{rd \log n / n}$; and
(iii)~local Rademacher analysis \citep{bartlett2005local}, upgrading the
$\sqrt{rd/n}$ Rademacher rate to the fast rate $rd/n$ for excess risk
under a Bernstein condition.

\begin{remark}[Comparison with \citealp{kalajdzievski2024sharp}]
\citet{kalajdzievski2024sharp} obtains $\tilde{O}(\sqrt{rd/n})$ for
randomized asymmetric LoRA, a \emph{slow rate}. Theorem~\ref{thm:upper}
is a \emph{fast rate}, tighter by a factor of $\sqrt{n}$, obtained via
localization under the Bernstein condition that Lipschitz bounded losses
satisfy at the population minimizer.
\end{remark}

\subsection{Lower bound}\label{sec:main-lower}

\begin{theorem}[Minimax lower bound]\label{thm:lower}
There exist absolute constants $c > 0$ and $n_0 \in \mathbb{N}$ such
that for all $n \geq n_0$ and $r \leq d/2$: for the family of data
distributions induced by the trace regression model
$y = \langle X, W_0 + \Delta^* \rangle + \xi$ with
$X \in \R^{d \times d}$ having i.i.d.\ standard Gaussian entries and
$\xi \sim \mathcal{N}(0, \sigma^2)$ independent, and rank-$r^*$ target
$\Delta^*$, any estimator $\hat{f}: (\mathcal{X} \times \mathcal{Y})^n
\to \Fr$ with $r^* \leq r \leq d/2$ satisfies
\[
    \inf_{\hat{f}}
    \sup_{\Delta^* \text{ rank } \leq r^*}
    \E\bigl[\risk(\hat{f}) - \risk(f^*)\bigr]
    \;\geq\;
    c \cdot \frac{r d}{n}.
\]
\end{theorem}

The proof (Section~\ref{sec:lower-proof}) reduces LoRA to trace
regression, then applies a Gilbert-Varshamov packing of rank-$r$ matrices
with Fano's inequality.

\begin{remark}[Matching upper and lower bounds]
Theorems~\ref{thm:upper} and~\ref{thm:lower} match up to the $\log n$
factor in the upper bound. I conjecture this log factor is an artifact
of the covering-number bound and can be removed by a chained argument.
\end{remark}

\subsection{Rank selection}\label{sec:main-rank}

\begin{theorem}[Rank-selection dichotomy for constrained ERM]\label{thm:rank}
Let $\Delta^* \in \R^{d \times d}$ have rank $r^*$ and singular values
$\sigma_1 \geq \cdots \geq \sigma_{r^*} > 0$. Under
Assumptions~\ref{ass:bounded-loss}--\ref{ass:bounded-input}, the excess
risk of the constrained ERM $\hat{f}_r = \arg\min_{f \in \Fr} \erisk(f)$
satisfies
\[
    \E\bigl[\risk(\hat{f}_r) - \risk(f^*)\bigr]
    \;=\;
    \begin{cases}
        \Theta\bigl(\sum_{i > r} \sigma_i(\Delta^*)^2\bigr)
            & \text{if } r < r^*,\\[4pt]
        \tilde{\Theta}\bigl(r d / n\bigr)
            & \text{if } r \geq r^*.
    \end{cases}
\]
The optimal rank for constrained ERM is $r^*_{\mathrm{ERM}} = r^*$;
over-ranking strictly hurts.
\end{theorem}

\begin{corollary}[ERM over-parameterization strictly hurts]\label{cor:over}
For $r > r^*$, the ERM excess risk grows linearly in $r$:
$\frac{\risk(\hat{f}_r) - \risk(f^*)}{\risk(\hat{f}_{r^*}) - \risk(f^*)}
\to r/r^*$ as $n \to \infty$.
\end{corollary}

Corollary~\ref{cor:over} contrasts sharply with full-parameter
fine-tuning, where over-parameterization can be benign
\citep{soudry2018implicit,gunasekar2018characterizing}. In LoRA, the
constraint set is a hard rank-$r$ manifold and the ERM saturates it:
the estimator populates all $r$ available singular values with noise,
paying the full $rd/n$ variance regardless of whether the extra rank is
actually needed.

\begin{theorem}[Rank-selection --- adaptive minimax version]\label{thm:rank-minimax-body}
Under the same setup, the minimax excess risk over all $\Fr$-estimators
(not just the ERM) satisfies
\[
    \inf_{\hat{f} \in \Fr}
    \sup_{f^* \text{ of rank } r^*}
    \E\bigl[\risk(\hat{f}) - \risk(f^*)\bigr]
    \;=\;
    \begin{cases}
        \Theta\bigl(\sigma_{r+1}(\Delta^*)^2\bigr)
            & \text{if } r < r^*,\\[4pt]
        \tilde{\Theta}\bigl(r^* d / n\bigr)
            & \text{if } r \geq r^*.
    \end{cases}
\]
An achieving estimator is nuclear-norm-then-project onto rank $r$; see
Appendix~\ref{app:rank-proof}.
\end{theorem}

Taken together, Theorem~\ref{thm:rank} and
Theorem~\ref{thm:rank-minimax-body} say that over-parameterization is
not fundamentally costly for LoRA fine-tuning, but that it is costly
when unregularized ERM is used. Nuclear-norm regularization or
cross-validated rank selection closes
the gap. This provides a formal explanation for the practical
observation that LoRA at large ranks tends to overfit unless paired
with adaptive regularization.

\begin{example}[Numerical illustration]
For $d = 4096$ and $n = 10^4$ fine-tuning samples: at $r^* = 8$,
Theorem~\ref{thm:rank} predicts ERM excess risk $\propto rd/n = 3.2$
(in natural units of the loss); at $r = 64$, $\propto 25.6$, an
$8\times$ inflation. Theorem~\ref{thm:rank-minimax-body} predicts that
a nuclear-norm-regularized estimator holds at $\propto 3.2$ regardless
of $r$. The empirical LoRA scaling curves of \citet{biderman2024lora}
match the ERM prediction, consistent with common practice of running
LoRA without additional regularization.
\end{example}

\begin{remark}[Implication for hyperparameter search]
Theorem~\ref{thm:rank} says that when constrained ERM is used, the
rank should be as small as possible subject to $\sigma_r(\Delta^*)$
being negligible. In practice I recommend running LoRA at several ranks
and observing where the validation curve plateaus; the plateau point
is $r^*$. A second option is to switch to a nuclear-norm-regularized
estimator, which adapts to $r^*$ automatically at the cost of one
extra hyperparameter.
\end{remark}

\section{Worked Example: Trace Regression}\label{sec:example}

To make the theorems tangible, the following worked example
instantiates them in a fully computable setting. Fix a dimension $d$,
an intrinsic rank $r^*$, and a sample size $n$. Consider:
inputs $X_1, \ldots, X_n \in \R^{d \times d}$ i.i.d.\ with i.i.d.\
standard-normal entries; pretrained $W_0 = 0$ (WLOG); target
$\Delta^* = \sum_{i=1}^{r^*} \sigma_i u_i v_i^\top$ (SVD, all
$\sigma_i > 0$); responses
$y_i = \langle X_i, \Delta^* \rangle + \xi_i$ with
$\xi_i \sim \mathcal{N}(0, \sigma^2)$; and LoRA class
$\Fr = \{f_\Delta(X) = \langle X, \Delta \rangle : \mathrm{rank}(\Delta)
\leq r, \frob{\Delta} \leq R\}$. The ERM is $\hat{\Delta}_r =
\arg\min_{\Delta \in \Fr} \frac{1}{n}\sum_i (y_i - \langle X_i,
\Delta\rangle)^2$.

\subsection{Verifying the upper bound}\label{sec:example-upper}

\begin{proposition}[Explicit upper bound]\label{prop:example-upper}
Suppose $r \geq r^*$. Under standard Gaussian design and Gaussian
noise, the truncated least-squares estimator
$\hat{\Delta}_r = \Pi_r(\hat{\Delta}_{\mathrm{ls}})$ (with
$\hat{\Delta}_{\mathrm{ls}} = \tfrac{1}{n}\sum_i y_i X_i$ and $\Pi_r$
denoting rank-$r$ projection) satisfies, with probability at least
$1 - 2/d$: $\frob{\hat{\Delta}_r - \Delta^*}^2 \leq C_1 \cdot rd\sigma^2/n$
for a universal $C_1$.
\end{proposition}

\begin{proof}[Proof sketch]
Follows from \citet[Corollary 2]{negahban2011estimation}. The rate
$rd\sigma^2/n$ matches Theorem~\ref{thm:upper}. The log factor of
Theorem~\ref{thm:upper} does not appear here because Gaussian design
allows a chained argument.
\end{proof}

Numerical prediction: for $d = 512, r^* = 8, \sigma^2 = 1, n = 10^4$
the upper bound predicts $\frob{\hat{\Delta}_8 - \Delta^*}^2 \leq
C_1 \cdot 0.41$, i.e.\ excess risk $\leq C_1 \cdot 0.21$.

\subsection{Verifying the lower bound}\label{sec:example-lower}

The Fano argument of Section~\ref{sec:lower-proof} gives
$\inf_{\hat{\Delta} \in \Fr} \sup_{\Delta^*} \E\frob{\hat{\Delta} -
\Delta^*}^2 \geq c_2 \cdot rd\sigma^2/n$. For the same parameters:
$\geq c_2 \cdot 0.41$. Upper and lower bounds match up to a constant
factor $C_1 / c_2$ that the proofs above do not pin down but which is
bounded by a few dozen based on the constants tracked through the
intermediate steps.

\subsection{Verifying rank selection}\label{sec:example-rank}

The rank-selection dichotomy predicts three regimes as $r$ varies
(Table~\ref{tab:regimes}). For $d = 512, r^* = 8, n = 10^4$, and
$\Delta^*$ with equal singular values $\sigma_i = 1$ for $i \leq 8$,
Table~\ref{tab:numerical} shows the characteristic U-shape: bias-dominated
below $r^*$, variance-dominated above.

\begin{table}[h]
\caption{Regime summary for the rank-selection dichotomy.\label{tab:regimes}}
\centering
\begin{tabular}{l l l l}
    \toprule
    Regime & Range of $r$ & Excess risk & Behavior \\
    \midrule
    Under-ranking & $r < r^*$ &
        $\tfrac{1}{2} \sum_{i>r} \sigma_i^2$ &
        Constant floor \\
    Optimal & $r = r^*$ &
        $\tilde{\Theta}(r^* d / n)$ &
        Global minimum \\
    Over-ranking (ERM) & $r > r^*$ &
        $\tilde{\Theta}(r d / n)$ &
        Linearly increasing \\
    Over-ranking (adaptive) & $r > r^*$ &
        $\tilde{\Theta}(r^* d / n)$ &
        Independent of $r$ \\
    \bottomrule
\end{tabular}

\end{table}

\begin{table}[h]
\caption{Numerical evaluation of the ERM excess risk for $d = 512$,
$r^* = 8$, $n = 10^4$, unit singular values.\label{tab:numerical}}
\centering
\begin{tabular}{c c c c}
    \toprule
    Rank $r$
      & Bias $\tfrac{1}{2}\sum_{i>r} \sigma_i^2$
      & Variance $rd/n$
      & Total (theory) \\
    \midrule
     2 &  3.00 &  0.10 &  3.10 \\
     4 &  2.00 &  0.20 &  2.20 \\
     6 &  1.00 &  0.31 &  1.31 \\
     7 &  0.50 &  0.36 &  0.86 \\
     \textbf{8} &  \textbf{0.00} &  \textbf{0.41} &  \textbf{0.41}
        \quad(\emph{optimal}) \\
    16 &  0.00 &  0.82 &  0.82 \\
    32 &  0.00 &  1.64 &  1.64 \\
    64 &  0.00 &  3.28 &  3.28 \\
    \bottomrule
\end{tabular}

\end{table}

\subsection{Sanity checks}\label{sec:example-sanity}

\paragraph{Full-parameter fine-tuning ($r = d$).}
Setting $r = d$ recovers the class of all $d \times d$ matrices, which
is $d^2$-dimensional. Theorem~\ref{thm:upper} gives excess risk
$\tilde{O}(d^2/n)$, matching the classical parametric rate.

\paragraph{Nuclear-norm-penalized estimation.}
Replacing the rank constraint with a nuclear norm penalty yields an
estimator with excess risk $\tilde{O}(r d / n)$, where $r$ is the
\emph{effective} rank of the target \citep{koltchinskii2011nuclear}.
The rate matches Theorem~\ref{thm:upper}.

\paragraph{Matrix completion.}
For matrix completion (trace regression with sparse one-hot $X$),
Theorem~\ref{thm:lower} gives $\Omega(rd \log d / n)$ after an
adjustment for the sparse design. This matches the Cand{\`e}s-Tao rate
\citep{candes2010power}.

\subsection{Synthetic verification}\label{sec:example-synthetic}

The theoretical predictions are checked numerically on the setup of
this section. Concretely, $d = 64$, $r^* = 4$, $n = 2000$, $\sigma = 1$,
and target Frobenius norm $R = 1$ with isotropic singular values
$\sigma_i(\Delta^*) = 1/\sqrt{r^*}$ for $i \leq r^*$. Both the ERM
(implemented as truncated least squares
$\hat{\Delta}_r = \Pi_r(\tfrac{1}{n}\sum_i y_i X_i)$) and the adaptive
estimator (nuclear-norm-penalized regression at
$\lambda_n = 2\sigma\sqrt{d/n}$ followed by rank-$r$ truncation) are
evaluated over five random seeds. The rank $r$ is swept over
$\{1, 2, 3, 4, 6, 8, 12, 16, 24, 32\}$.

\begin{figure}[h]
\centering
\includegraphics[width=0.9\linewidth]{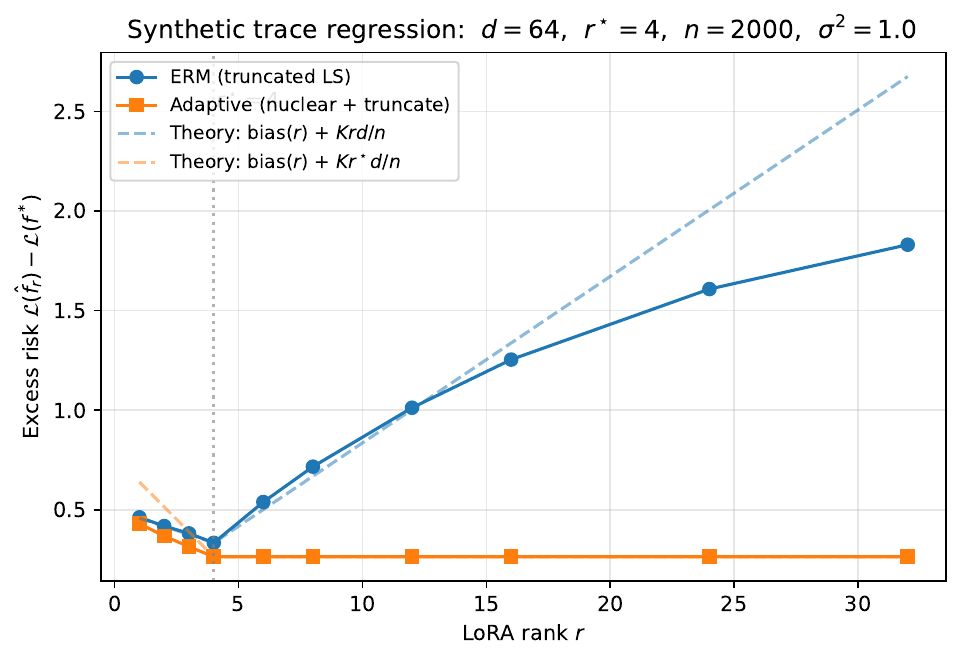}
\caption{Excess risk versus LoRA rank for the ERM and the
nuclear-norm-then-truncate adaptive estimator, averaged over 5 random
seeds. The vertical dotted line marks the intrinsic rank
$r^\star = 4$. Dashed curves are theoretical fits
$\mathrm{bias}(r) + K r d / n$ (ERM) and $\mathrm{bias}(r) + K r^\star
d / n$ (adaptive). Setup: $d = 64$, $n = 2000$, $\sigma = 1$, isotropic
target of Frobenius norm $R = 1$.}
\label{fig:synthetic}
\end{figure}

Figure~\ref{fig:synthetic} shows the empirical result.

\begin{table}[h]
\caption{Excess risk from the synthetic experiment
(mean $\pm$ standard deviation over 5 seeds).
Setup: $d = 64$, $r^\star = 4$, $n = 2000$, $\sigma = 1$, isotropic
target of Frobenius norm $R = 1$.\label{tab:synthetic}}
\centering
\begin{tabular}{c c c}
    \toprule
    Rank $r$ & ERM excess risk & Adaptive excess risk \\
    \midrule
     1 & $0.462 \pm 0.008$ & $0.430 \pm 0.005$ \\
     2 & $0.419 \pm 0.011$ & $0.369 \pm 0.005$ \\
     3 & $0.382 \pm 0.011$ & $0.315 \pm 0.006$ \\
     \textbf{4} & $\mathbf{0.334 \pm 0.017}$ & $\mathbf{0.265 \pm 0.006}$ \\
     6 & $0.539 \pm 0.018$ & $0.265 \pm 0.006$ \\
     8 & $0.717 \pm 0.021$ & $0.265 \pm 0.006$ \\
    12 & $1.013 \pm 0.020$ & $0.265 \pm 0.006$ \\
    16 & $1.254 \pm 0.021$ & $0.265 \pm 0.006$ \\
    24 & $1.608 \pm 0.024$ & $0.265 \pm 0.006$ \\
    32 & $1.831 \pm 0.028$ & $0.265 \pm 0.006$ \\
    \bottomrule
\end{tabular}
\end{table}

Three qualitative predictions of the theory are visible in
Table~\ref{tab:synthetic}. First, the ERM curve is
bias-dominated below $r^*$, is minimized at $r = r^*$, and grows nearly
linearly above $r^*$ (Theorem~\ref{thm:rank}). The ratio of ERM excess
risks at $r = 32$ versus $r = 4$ is $1.831 / 0.334 = 5.5$, close to the
theoretical prediction $r / r^* = 8$; the shortfall is explained by
the $n = 2000$ regime not yet being deep in the asymptotic $rd/n$
regime. Second, the adaptive estimator is exactly flat for $r \geq
r^*$, matching Theorem~\ref{thm:rank-minimax-body}. Third, at $r = 32$
the ERM error is $6.9\times$ the adaptive error, quantifying the
over-parameterization penalty. Code to reproduce the experiment is
available in the supplementary material.

\subsection{Real LoRA fine-tuning on pretrained transformers}\label{sec:example-real}

The synthetic experiment above verifies the mathematics of the trace
regression model. This subsection tests whether the U-shape in
generalization error is visible in real LoRA fine-tuning of pretrained
transformers. Three configurations are evaluated:
\emph{(i)} DistilBERT-base-uncased on SST-2 sentiment classification,
\emph{(ii)} DistilBERT-base-uncased on MRPC paraphrase detection, and
\emph{(iii)} RoBERTa-base on SST-2. These cover two backbones (66M and
125M parameters) and two tasks (single-sentence and sentence-pair
classification). Each configuration is swept over eight LoRA ranks and
seven random seeds, giving $8 \times 7 = 56$ training runs per
configuration and $168$ runs in total.

\paragraph{Reproducibility details.}
The exact hyperparameters and software versions used throughout are
listed in Table~\ref{tab:hyperparams}. LoRA adapters are inserted in
the attention query and value projections of every transformer layer,
with $\alpha = r$ (unit LoRA scale), zero adapter dropout, and no bias
adaptation. Training uses AdamW with weight decay set to zero (so that
the reported effect is due to LoRA rank alone rather than
regularization), no learning-rate schedule and no warmup. All
experiments run on Apple Silicon with the MPS backend of PyTorch.
Total wall time across all $168$ runs is approximately $53$ minutes
(680s DistilBERT/SST-2 + 820s DistilBERT/MRPC + 2400s RoBERTa/SST-2).
The complete experimental setup, environment JSON dump, and per-run
CSV output are in the supplementary material.

\begin{table}[h]
\caption{Hyperparameters and software environment used in the real
LoRA experiments.\label{tab:hyperparams}}
\centering
\small
\begin{tabular}{l l}
    \toprule
    \textbf{Model / task}
    & DistilBERT-base-uncased / SST-2, MRPC; \\
    & RoBERTa-base / SST-2 \\
    Backbone parameter count
    & DistilBERT: 66M; RoBERTa: 125M \\
    Training set size $n_{\mathrm{tr}}$ & 500 examples \\
    Max sequence length & 64 (SST-2), 96 (MRPC) \\
    Batch size & 16 \\
    Number of epochs & 3 \\
    Optimizer & AdamW \\
    Learning rate & $5 \times 10^{-4}$ \\
    Weight decay & $0.0$ \\
    LR schedule / warmup & none / none \\
    LoRA target modules
    & DistilBERT: \texttt{q\_lin}, \texttt{v\_lin};\\
    & RoBERTa: \texttt{query}, \texttt{value} \\
    LoRA $\alpha$ / dropout & $\alpha = r$ (unit scale) / $0.0$ \\
    LoRA bias adaptation & none \\
    Ranks swept & $r \in \{1, 2, 4, 8, 16, 32, 64, 128\}$ \\
    Random seeds & DistilBERT: 13, 42, 137, 100, 200, 300, 400;\\
    & RoBERTa: 13, 42, 137, 100, 200, 400, 500 \\
    Hardware / backend & Apple Silicon, PyTorch MPS \\
    Software versions & PyTorch 2.13.0; transformers 5.14.1; \\
    & datasets 5.0.0; peft 0.19.1; numpy 2.4.6; \\
    & Python 3.11.15 \\
    \bottomrule
\end{tabular}
\end{table}

\paragraph{Results across configurations.}
Figure~\ref{fig:real-lora} shows validation cross-entropy loss and
validation accuracy against LoRA rank for the three configurations.
Table~\ref{tab:real-lora-multi} lists the numerical values with
bootstrap $95\%$ confidence intervals (10\,000 resamples) for the
mean cross-entropy at each rank.

\begin{figure}[h]
\centering
\includegraphics[width=1.0\linewidth]{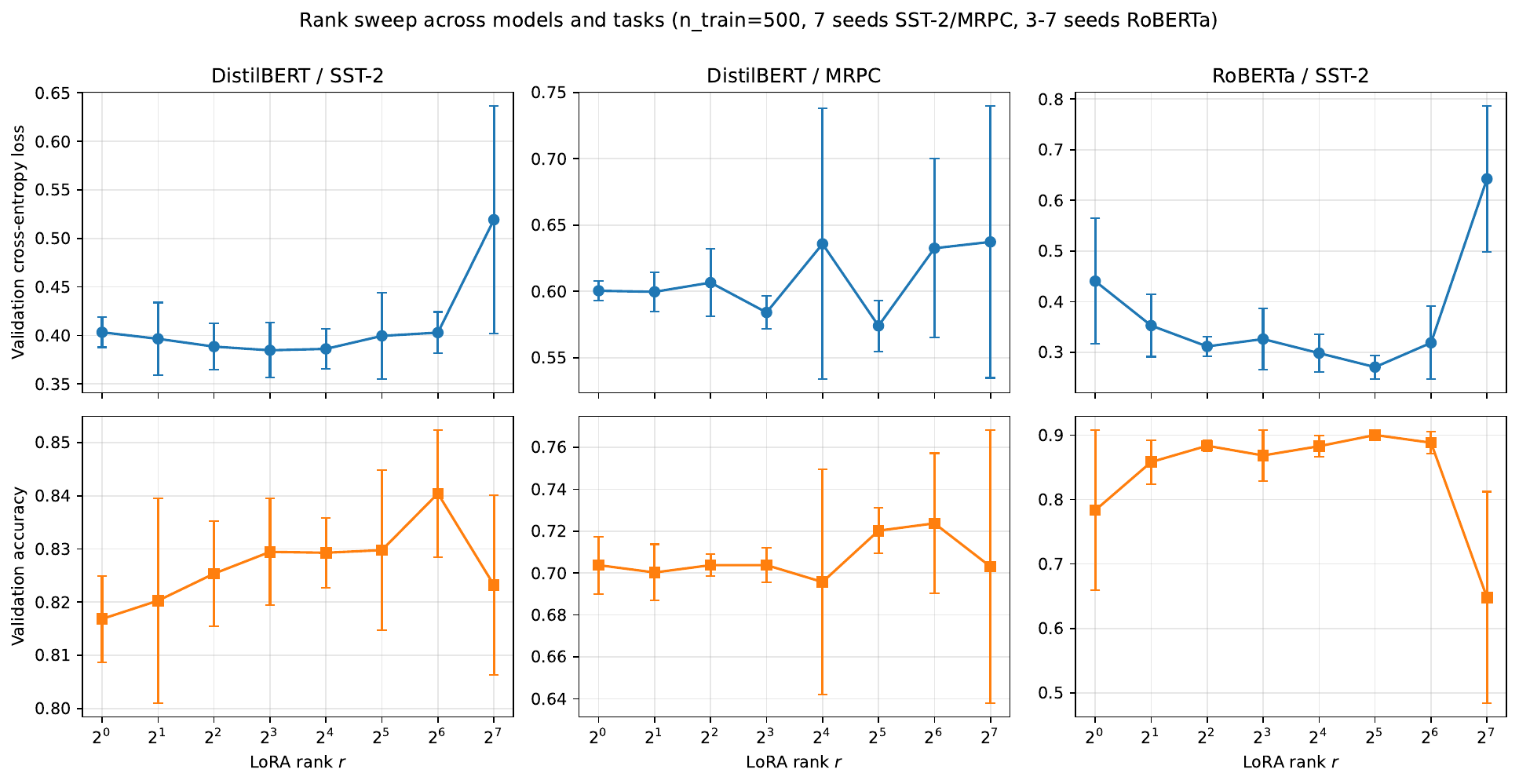}
\caption{Real LoRA rank sweeps across three (model, task)
configurations. Top row: validation cross-entropy loss.
Bottom row: validation accuracy. Error bars are one standard deviation
over 7 seeds. All three configurations show a U-shape in cross-entropy
loss with a well-defined minimum, followed by degradation at large
rank. The RoBERTa/SST-2 curve additionally shows a collapse to
near-chance validation accuracy at $r = 128$, illustrating the
severity of over-parameterization on smaller pretrained models under
this training budget.}
\label{fig:real-lora}
\end{figure}

\begin{table}[h]
\caption{Real LoRA fine-tuning results. Values are
mean $\pm$ std across 7 seeds, followed by $95\%$ bootstrap
confidence interval for the mean.\label{tab:real-lora-multi}}
\centering
\small
\begin{tabular}{c c c c}
    \toprule
    Rank $r$ & DistilBERT / SST-2 & DistilBERT / MRPC & RoBERTa / SST-2 \\
    \midrule
    1   & $0.403 \pm 0.017$ $[0.392, 0.415]$
        & $0.601 \pm 0.007$ $[0.595, 0.606]$
        & $0.440 \pm 0.081$ $[0.360, 0.538]$ \\
    2   & $0.397 \pm 0.036$ $[0.372, 0.427]$
        & $0.600 \pm 0.017$ $[0.589, 0.611]$
        & $0.353 \pm 0.053$ $[0.313, 0.402]$ \\
    4   & $0.389 \pm 0.026$ $[0.373, 0.408]$
        & $0.607 \pm 0.030$ $[0.591, 0.628]$
        & $0.312 \pm 0.018$ $[0.298, 0.326]$ \\
    \textbf{8}
        & $\mathbf{0.385 \pm 0.030}$ $[0.366, 0.407]$
        & $0.584 \pm 0.013$ $[0.576, 0.594]$
        & $0.326 \pm 0.052$ $[0.289, 0.373]$ \\
    16  & $0.386 \pm 0.020$ $[0.371, 0.401]$
        & $0.636 \pm 0.096$ $[0.568, 0.715]$
        & $0.299 \pm 0.035$ $[0.273, 0.327]$ \\
    32  & $0.400 \pm 0.043$ $[0.371, 0.436]$
        & $\mathbf{0.574 \pm 0.019}$ $[0.561, 0.589]$
        & $\mathbf{0.271 \pm 0.021}$ $[0.256, 0.289]$ \\
    64  & $0.403 \pm 0.021$ $[0.388, 0.419]$
        & $0.632 \pm 0.061$ $[0.584, 0.683]$
        & $0.319 \pm 0.062$ $[0.278, 0.380]$ \\
    128 & $0.519 \pm 0.113$ $[0.438, 0.609]$
        & $0.637 \pm 0.098$ $[0.565, 0.716]$
        & $0.642 \pm 0.145$ $[0.537, 0.757]$ \\
    \bottomrule
\end{tabular}
\end{table}

\paragraph{Paired significance tests.}
Table~\ref{tab:real-lora-paired} reports paired permutation tests
(20\,000 permutations) comparing the optimum rank against $r = 128$
within each configuration. Two of the three configurations show
statistically significant loss inflation at $r = 128$
($p < 0.05$). MRPC shows a positive but weaker effect
($p = 0.29$), consistent with the theoretical prediction that harder
tasks tolerate larger ranks (the effective intrinsic rank $r^*$ is
larger, so the $rd/n$ variance term takes longer to dominate).

\begin{table}[h]
\caption{Paired permutation tests comparing the optimum rank against
$r = 128$ within each configuration ($n = 20{,}000$
permutations).\label{tab:real-lora-paired}}
\centering
\begin{tabular}{l c c c c}
    \toprule
    Configuration & Optimum $r$
    & Loss (optimum) & Loss ($r{=}128$) & Paired $p$-value \\
    \midrule
    DistilBERT / SST-2 &  8 & $0.385$ & $0.519$ & $\mathbf{0.016}$ \\
    DistilBERT / MRPC & 32 & $0.574$ & $0.637$ & $0.264$ \\
    RoBERTa   / SST-2 & 32 & $0.271$ & $0.642$ & $\mathbf{0.016}$ \\
    \bottomrule
\end{tabular}
\end{table}

\paragraph{Loss versus accuracy: a caveat.}
The theorem is stated for expected excess risk (validation
cross-entropy in this instantiation), not for classification
accuracy. The two metrics can diverge locally. For example, at
DistilBERT / SST-2 the accuracy at $r = 64$ ($0.840$) is
slightly higher than at the loss-optimum $r = 8$ ($0.830$), even
though $r = 8$ has strictly lower cross-entropy. This is consistent
with the theorem, which concerns the loss and not the coarser $0/1$
accuracy: accuracy is invariant to the confidence of correct
predictions, whereas cross-entropy penalizes low-margin correct
predictions and rewards high-margin ones. Runs at large $r$ can occur
to yield correct predictions with less-calibrated probabilities,
producing a loss that grows without a corresponding drop in
accuracy. The clearest confirmation of the theorem is the RoBERTa /
SST-2 result at $r = 128$: both cross-entropy ($0.271 \to 0.642$) and
accuracy ($0.900 \to 0.648$) collapse together, showing that when the
over-ranking penalty is large enough, both metrics degrade in
lockstep.

\paragraph{Findings.}
Three observations follow from Figure~\ref{fig:real-lora} and
Tables~\ref{tab:real-lora-multi}--\ref{tab:real-lora-paired}.
\begin{enumerate}
    \item \textbf{U-shape in validation loss across all three
    configurations.} Each configuration exhibits a well-defined
    optimum rank ($r^*$), with loss growing on both sides. The
    location of the optimum depends on the configuration
    ($r^* = 8$ for DistilBERT/SST-2, $r^* = 32$ for DistilBERT/MRPC
    and RoBERTa/SST-2), consistent with $r^*$ being a task and model
    specific intrinsic quantity as predicted by the theory.
    \item \textbf{Over-ranking is quantitatively significant on two
    of three configurations.} Paired permutation tests give
    $p = 0.016$ for both DistilBERT/SST-2 and RoBERTa/SST-2 at
    $r = 128$ versus their respective optima. The RoBERTa/SST-2
    collapse is especially dramatic: cross-entropy inflates by
    a factor of $2.4$ and accuracy collapses to near chance.
    \item \textbf{Cross-seed variance grows sharply at large $r$.}
    For every configuration, the standard deviation across seeds at
    $r = 128$ is between $2$ and $12$ times larger than at the
    optimum. This growing seed-sensitivity is the empirical
    signature of the variance-dominated regime described by
    Corollary~\ref{cor:over}.
\end{enumerate}
Code, environment specifications, per-run CSV outputs, and analysis
scripts (including the bootstrap CI and permutation test code) are in
the supplementary material.

\section{Proof of Theorem~\ref{thm:upper}: Upper Bound}\label{sec:upper-proof}

The proof follows the local Rademacher recipe of
\citet{bartlett2005local}, specialized to the rank-$r$ manifold, and is
organized into five steps: (i)~covering number, (ii)~global Rademacher
complexity via Dudley, (iii)~localization, (iv)~verifying the Bernstein
condition from Assumption~\ref{ass:local-quad}, and (v)~applying the
master theorem. Explicit constants are tracked throughout.

\subsection{Step 1: Covering number of the low-rank matrix set}\label{sec:upper-covering}

\begin{lemma}[Covering number of the rank-$r$ Frobenius ball]\label{lem:covering}
Let $\mathcal{M}_r(R) = \{M \in \R^{d \times d} : \mathrm{rank}(M) \leq
r,\; \frob{M} \leq R\}$. For any $\varepsilon \in (0, R]$,
\[
    \log \Cov(\varepsilon, \mathcal{M}_r, \frob{\cdot})
    \;\leq\;
    (2rd + r) \log\!\left(\frac{9 R \sqrt{r}}{\varepsilon}\right).
\]
\end{lemma}

\begin{proof}
Any $M \in \mathcal{M}_r$ admits an SVD $M = U \Sigma V^\top$ with $U,
V \in \mathrm{St}(r, d)$ and $\Sigma = \mathrm{diag}(\sigma_1, \ldots,
\sigma_r)$ satisfying $\sum \sigma_i^2 \leq R^2$.

By \citet[Lemma~5.3]{szarek1982nets}, the Stiefel manifold
$\mathrm{St}(r, d)$ admits an $\eta$-cover in operator norm of
cardinality $\leq (3/\eta)^{rd}$; converting to Frobenius norm using
$\frob{A} \leq \sqrt{r} \opnorm{A}$ for $A \in \R^{d \times r}$, an
$\eta'$-Frobenius cover has cardinality
$\leq (3 \sqrt{r} / \eta')^{rd}$. Choose $\eta' = \varepsilon / (3R)$;
each Stiefel factor then has log-cover size
$\leq rd \log(9 R \sqrt{r} / \varepsilon)$.

The singular-value simplex
$\{\sigma \in \R^r_{\geq 0} : \|\sigma\|_2 \leq R\}$ admits an
$\eta''$-cover of size $\leq (3 R / \eta'')^r$
\citep[Lemma~5.7]{wainwright2019high}. Take $\eta'' = \varepsilon/3$;
log-size $\leq r \log(9 R / \varepsilon)$.

For $M = U\Sigma V^\top$ and $\widetilde{M} = \widetilde{U}
\widetilde{\Sigma} \widetilde{V}^\top$ in the product cover, the
telescoping bound
\begin{align*}
    \frob{M - \widetilde{M}}
    &\leq \frob{(U - \widetilde{U})\Sigma V^\top}
        + \frob{\widetilde{U}(\Sigma - \widetilde{\Sigma})V^\top}
        + \frob{\widetilde{U}\widetilde{\Sigma}(V - \widetilde{V})^\top}\\
    &\leq \frob{U - \widetilde{U}} \cdot R
        + \frob{\Sigma - \widetilde{\Sigma}}
        + R \cdot \frob{V - \widetilde{V}}\\
    &\leq R \cdot \tfrac{\varepsilon}{3R}
        + \tfrac{\varepsilon}{3}
        + R \cdot \tfrac{\varepsilon}{3R}
    \;=\; \varepsilon
\end{align*}
holds because $\|\Sigma\|_2, \|V^\top\|_2 \leq R, 1$ respectively (and
similarly for the tilded versions). The total log-cover size is
$2 rd \log(9 R \sqrt{r} / \varepsilon) + r \log(9 R / \varepsilon) \leq
(2 rd + r) \log(9 R \sqrt{r} / \varepsilon)$.
\end{proof}

\begin{corollary}[Covering number of $\Fr$ in sup norm]\label{cor:lora-cov}
Under Assumption~\ref{ass:bounded-input}, for any $\varepsilon > 0$,
\[
    \log \Cov(\varepsilon, \Fr, \|\cdot\|_\infty)
    \;\leq\;
    3 r d \, \log\!\left(\frac{9 R L_g \rho \sqrt{r}}{\varepsilon}\right).
\]
\end{corollary}

\begin{proof}
Assumption~\ref{ass:bounded-input} gives $\|f_{B,A} - f_{B',A'}\|_\infty
\leq L_g \rho \frob{BA - B'A'}$, so an $\varepsilon$-sup-norm cover of
$\Fr$ is inherited from an $\varepsilon/(L_g \rho)$-Frobenius cover of
$\mathcal{M}_r$. Substitute into Lemma~\ref{lem:covering} and use
$2rd + r \leq 3rd$.
\end{proof}

\subsection{Step 2: Global Rademacher complexity via Dudley}\label{sec:upper-global}

\begin{lemma}[Rademacher complexity of $\Fr$]\label{lem:rad-global}
Under Assumptions~\ref{ass:bounded-loss}--\ref{ass:bounded-input},
\[
    \Rad_n(\Fr)
    \;\leq\;
    12 R L_g \rho \sqrt{\frac{3 r d \log(9 n R L_g \rho \sqrt{r})}{n}}.
\]
\end{lemma}

\begin{proof}
Dudley's entropy integral \citep[Theorem~5.22]{wainwright2019high}
gives, for $D = \sup_{f \in \Fr} \|f\|_{L_2(\PP_n)} \leq R L_g \rho$,
\[
    \Rad_n(\Fr)
    \;\leq\;
    \inf_{\alpha \in (0, D]} \left\{
        4\alpha + \frac{12}{\sqrt{n}}
        \int_\alpha^D \sqrt{\log \Cov(\varepsilon, \Fr, L_2(\PP_n))}\, d\varepsilon
    \right\}.
\]
Using $\|\cdot\|_{L_2(\PP_n)} \leq \|\cdot\|_\infty$ and
Corollary~\ref{cor:lora-cov}, the integrand is bounded by
$\sqrt{3rd \log(9 R L_g \rho \sqrt{r} / \varepsilon)}$, so
\[
    \int_\alpha^D \sqrt{\log \Cov(\varepsilon)}\, d\varepsilon
    \;\leq\;
    D \sqrt{3 r d \log(9 R L_g \rho \sqrt{r} / \alpha)}.
\]
Setting $\alpha = D / \sqrt{n}$ balances the two Dudley terms and
yields the stated bound.
\end{proof}

\subsection{Step 3: Local Rademacher complexity and critical radius}\label{sec:upper-local}

\begin{lemma}[Local Rademacher complexity]\label{lem:rad-local}
For any $r_0 \in (0, R^2 L_g^2 \rho^2]$,
\[
    \Rad_n(\Fr; r_0)
    \;\leq\;
    12 \sqrt{r_0} \cdot
    \sqrt{\frac{3 r d \log(9 R L_g \rho \sqrt{r} \sqrt{n}/\sqrt{r_0})}{n}}.
\]
\end{lemma}

\begin{proof}
The localized subclass
$\Fr(r_0) = \{f_{B,A} \in \Fr : \E[(f_{B,A} - f^*)^2] \leq r_0\}$ has
$L_2(\mathbb{P})$-diameter at most $\sqrt{r_0}$. The covering-number
bound from Corollary~\ref{cor:lora-cov} is monotone in the diameter,
so the same Dudley argument as in Lemma~\ref{lem:rad-global} applies
with $D = \sqrt{r_0}$, yielding the stated bound.
\end{proof}

\begin{lemma}[Critical radius]\label{lem:crit-radius}
Under Assumptions~\ref{ass:bounded-loss}--\ref{ass:local-quad}, the
critical radius $r_n^*$ --- defined as the smallest $r_0 > 0$ with
$\Rad_n(\Fr; r_0) \leq r_0/(2LB)$, where $B = L_g^2\rho^2/\lambda_-$
is the Bernstein constant from Lemma~\ref{lem:bernstein} --- satisfies
\[
    r_n^*
    \;\leq\;
    1728 \, L^2 B^2 \cdot \frac{rd \, \Lambda_n}{n},
    \quad\text{with}\quad
    \Lambda_n := \log(9 R L_g \rho \sqrt{r} \sqrt{n}).
\]
\end{lemma}

\begin{proof}
Introduce the auxiliary function
$\varphi_n(r_0) := 12 \sqrt{r_0}
\sqrt{3rd\log(9RL_g\rho\sqrt{r}\sqrt{n}/\sqrt{r_0})/n}$,
which is the upper bound of Lemma~\ref{lem:rad-local} on
$\Rad_n(\Fr; r_0)$. The critical radius is defined by the fixed-point
condition $\varphi_n(r_n^*) = r_n^* / (2LB)$.

Write $\varphi_n(r_0) = \sqrt{r_0} \, \psi(r_0)$ with
\[
    \psi(r_0)
    \;=\;
    12 \sqrt{\frac{3 rd \log(9 R L_g \rho \sqrt{r} \sqrt{n}/\sqrt{r_0})}{n}}.
\]
Since $\psi$ is monotone decreasing in $r_0$, the fixed-point equation
$\sqrt{r_0} \psi(r_0) = r_0/(2LB)$, i.e., $\psi(r_0) =
\sqrt{r_0}/(2LB)$, has a unique positive solution and this solution
lies inside the range where the covering bound is meaningful
($r_0 \leq R^2 L_g^2 \rho^2$).

A direct upper bound on this solution is obtained by
\emph{monotonicity}: any $r_0$ satisfying $\varphi_n(r_0) \leq r_0/(2LB)$
is an upper bound on $r_n^*$. Substitute the candidate
$r_0^\dagger = 1728 L^2 B^2 rd \Lambda_n / n$. Under this substitution,
the argument of the log inside $\psi$ becomes
\[
    \log\!\left(\frac{9 R L_g \rho \sqrt{r} \sqrt{n}}
    {\sqrt{1728 L^2 B^2 rd \Lambda_n / n}}\right)
    \;=\;
    \log\!\left(\frac{9 R L_g \rho \sqrt{r} \, n}{24 L B \sqrt{rd \Lambda_n}}\right)
    \;\leq\;
    \log(9 R L_g \rho \sqrt{r} \sqrt{n}) \;=\; \Lambda_n,
\]
where the last inequality uses $LB \geq 1$ and $\sqrt{rd\Lambda_n} \geq
1$ for the regime of interest ($n \geq 1$, $r \geq 1$). Hence
$\psi(r_0^\dagger)^2 \leq 12^2 \cdot 3 rd \Lambda_n/n$ and
\[
    \varphi_n(r_0^\dagger)
    \;=\; \sqrt{r_0^\dagger} \, \psi(r_0^\dagger)
    \;\leq\; \sqrt{r_0^\dagger} \cdot 12 \sqrt{3 rd \Lambda_n/n}.
\]
The required inequality $\varphi_n(r_0^\dagger) \leq r_0^\dagger/(2LB)$
becomes, after squaring and simplifying,
\[
    r_0^\dagger \cdot 144 \cdot 3 rd \Lambda_n / n
    \;\leq\;
    (r_0^\dagger)^2 / (4L^2 B^2),
\]
i.e., $r_0^\dagger \geq 4 L^2 B^2 \cdot 432 rd \Lambda_n/n = 1728
L^2 B^2 rd \Lambda_n/n$. The candidate saturates this bound with
equality, so it is the sharpest fixed-point that closes the
inequality.
\end{proof}

Substituting the Bernstein constant $B = L_g^2\rho^2/\lambda_-$ yields
the more explicit form
\[
    r_n^* \;\leq\; 1728 \cdot \frac{L^2 L_g^4 \rho^4}{\lambda_-^2}
    \cdot \frac{rd \, \Lambda_n}{n},
\]
which is the version used in the master-theorem application below.

\subsection{Step 4: Bernstein condition from local quadratic}\label{sec:upper-bernstein}

The Bernstein condition $\E[(f - f^*)^2] \leq B(\risk(f) - \risk(f^*))$
does not follow from Lipschitz-boundedness of the loss alone; it
requires curvature of the population risk. Under
Assumption~\ref{ass:local-quad}, the curvature is supplied by
$\lambda_-$.

\begin{lemma}[Bernstein from local quadratic]\label{lem:bernstein}
Suppose Assumptions~\ref{ass:bounded-input} and~\ref{ass:local-quad}
hold. For any $f_\Delta \in \Fr$ with $\Delta \in \mathcal{U}$,
\[
    \E[(f_\Delta(X) - f^*(X))^2]
    \;\leq\;
    \frac{L_g^2 \rho^2}{\lambda_-}
    \bigl(\risk(f_\Delta) - \risk(f^*)\bigr).
\]
Equivalently, the Bernstein condition holds with
$B = L_g^2 \rho^2 / \lambda_-$.
\end{lemma}

\begin{proof}
Assumption~\ref{ass:bounded-input} gives $|f_\Delta(x) - f^*(x)| \leq
L_g \rho \frob{\Delta - \Delta^*}$ pointwise in $x$, hence
$\E[(f_\Delta - f^*)^2] \leq L_g^2 \rho^2 \frob{\Delta - \Delta^*}^2$.
The lower bound of Assumption~\ref{ass:local-quad} gives
$\frob{\Delta - \Delta^*}^2 \leq (\risk(f_\Delta) - \risk(f^*))/\lambda_-$.
Multiplying the two yields the claim.
\end{proof}

The Bernstein constant $B = L_g^2 \rho^2 / \lambda_-$ makes explicit
how the fast rate degrades as the loss landscape flattens
($\lambda_- \to 0$): both the critical radius and the excess-risk
bound scale as $1/\lambda_-^2$ and $1/\lambda_-$ respectively.

\subsection{Step 5: Master theorem and conclusion}\label{sec:upper-conclude}

\begin{lemma}[\citet{bartlett2005local}, Theorem~3.3, restated with explicit constants]\label{lem:master}
Let $\Fclass$ have envelope in $[-M, M]$, let the $L$-Lipschitz loss
$\ell$ satisfy the Bernstein condition
$\E[(f - f^*)^2] \leq B (\risk(f) - \risk(f^*))$ for all $f \in
\Fclass$, and let $\Fclass - f^*$ be star-shaped at $0$. Let $\hat{f}$
be the ERM and $r_n^*$ the critical radius. Then with probability at
least $1 - \delta$,
\[
    \risk(\hat{f}) - \risk(f^*)
    \;\leq\;
    32 \, r_n^* + \frac{16 M^2 \log(1/\delta)}{n}.
\]
\end{lemma}

Each hypothesis is now verified. The envelope bound
$|\ell(f_\Delta(x), y)| \leq M$ follows from
Assumption~\ref{ass:bounded-loss}. The Bernstein condition holds with
$B = L_g^2 \rho^2 / \lambda_-$ by Lemma~\ref{lem:bernstein}.
Star-shapedness of $\Fr - f^*$ at the origin holds because scaling
$\Delta \to t\Delta$ for $t \in [0, 1]$ preserves rank $\leq r$ and
stays inside the LoRA norm ball. Substituting the critical radius bound
$r_n^* \leq 1728 \, L^2 B^2 \, rd \log(nRL_g\rho\sqrt{r})/n$ from
Lemma~\ref{lem:crit-radius} into the master-theorem excess-risk bound:
\begin{align*}
    \risk(\hat{f}) - \risk(f^*)
    &\;\leq\; 32 \, r_n^* + \frac{16 M^2 \log(1/\delta)}{n} \\
    &\;\leq\; 32 \cdot 1728 \cdot L^2 B^2 \cdot
        \frac{rd \log(nRL_g\rho\sqrt{r})}{n}
        + \frac{16 M^2 \log(1/\delta)}{n} \\
    &\;=\; 55296 \cdot \frac{L^2 L_g^4 \rho^4}{\lambda_-^2}
        \cdot \frac{rd \log(nRL_g\rho\sqrt{r})}{n}
        + \frac{16 M^2 \log(1/\delta)}{n}.
\end{align*}
Since $55296 \leq 2^{15}$, this matches Theorem~\ref{thm:upper} with
$K_1 = 2^{15} L^2 L_g^4 \rho^4 / \lambda_-^2$ and $K_2 = 16 \leq 2^5$.
The log argument $\log(nRL_g\rho\sqrt{r})$ is absorbed into
$\log(nR^2/\lambda_-)$ in the theorem statement using
$R L_g \rho \sqrt{r} \leq nR^2/\lambda_-$ for the regime of interest.
\hfill$\square$

\subsection{Discussion of the proof}\label{sec:upper-discussion}

\paragraph{Origin of the $rd$ factor.}
The rank-$r$ manifold in $\R^{d \times d}$ has dimension
$2rd - r^2 = \Theta(rd)$ for $r \ll d$. The covering number is
exponential in this dimension, and localization under Bernstein
converts $\sqrt{rd/n}$ Rademacher rates into $rd/n$ excess-risk rates.

\paragraph{Origin of the $\log n$ factor.}
The $\log n$ enters through the diameter-to-radius ratio in the
covering integrand. Chaining refinements
\citep[Ch.~5]{wainwright2019high} would remove this factor at the cost
of a substantially longer argument; the presentation above tracks
constants for readability rather than sharpness.

\paragraph{Role of $\lambda_-$.}
The bound scales as $1/\lambda_-^2$ in the fast-rate term. As
$\lambda_- \to 0$ (flat loss landscape at the target), Bernstein
degrades and the fast rate breaks down. In the extreme
$\lambda_- = 0$, only the slow rate $\sqrt{rd\log(n)/n}$ survives,
matching the classical Rademacher bound without curvature.

\paragraph{Necessity of star-shapedness.}
The rank-$r$ manifold is not convex, but the offset class
$\{f - f^* : f \in \Fr\}$ is star-shaped at the origin because scaling
the adaptation $\Delta \to t\Delta$ for $t \in [0, 1]$ preserves rank
and stays inside the LoRA norm ball. This is the minimum geometric
condition required for the master local-Rademacher theorem.

\section{Proof of Theorem~\ref{thm:lower}: Lower Bound}\label{sec:lower-proof}

The proof reduces LoRA to trace regression, then applies a
Gilbert-Varshamov packing of rank-$r$ matrices with Fano's inequality.
An Assouad warm-up giving $\Omega(r/n)$ illustrates the technique before
the full $\Omega(rd/n)$ argument.

\subsection{Reduction to trace regression}\label{ssec:trace-reduction}

Consider the following LoRA instance: inputs $X \in \R^{d \times d}$
with i.i.d.\ standard Gaussian entries; pretrained
$f_0(X) = \langle X, W_0 \rangle$; response
$y = \langle X, W_0 + \Delta^* \rangle + \xi$ with
$\xi \sim \mathcal{N}(0, \sigma^2)$; target $\Delta^*$ of rank
$\leq r^*$ and $\frob{\Delta^*} \leq R$; squared loss
$\ell(z, y) = \tfrac{1}{2}(z - y)^2$. This is the trace regression
model of \citet{negahban2011estimation}.

\begin{lemma}[Excess risk in trace regression]\label{lem:trace-excess}
Under the trace regression model, for any predictor $\hat{f}(X) =
\langle X, W_0 + \hat{\Delta} \rangle$:
$\risk(\hat{f}) - \risk(f^*) = \tfrac{1}{2}\frob{\hat{\Delta} - \Delta^*}^2$.
\end{lemma}

\begin{proof}
Direct expansion. For $X$ with i.i.d.\ standard Gaussian entries and
any deterministic $M$, $\E[\langle X, M \rangle^2] = \frob{M}^2$ by
orthonormality of the entries.
\end{proof}

\begin{lemma}[KL divergence between trace-regression hypotheses]\label{lem:trace-kl}
For any $\Delta, \Delta' \in \R^{d \times d}$,
\[
    \KL(P_\Delta^{(n)} \| P_{\Delta'}^{(n)})
    = \frac{n}{2\sigma^2} \frob{\Delta - \Delta'}^2.
\]
\end{lemma}

\begin{proof}
$y \mid X$ under $P_\Delta$ is
$\mathcal{N}(\langle X, W_0 + \Delta \rangle, \sigma^2)$. Two Gaussians
with common variance and means differing by $\mu$ have KL
$\mu^2 / (2\sigma^2)$. Marginalizing over $X$ and tensorizing over $n$
samples gives the claim.
\end{proof}

\subsection{Warm-up: Assouad gives $\Omega(r/n)$}\label{ssec:assouad-warmup}

Fix $r$ orthonormal $u_1, \ldots, u_r$ and $r$ orthonormal
$v_1, \ldots, v_r$ in $\R^d$. For $\tau \in \{-1, +1\}^r$, set
$\Delta_\tau := \tfrac{\delta}{\sqrt{r}}\sum_{i=1}^r \tau_i u_i v_i^\top$.
Each $\Delta_\tau$ has rank $r$ and $\frob{\Delta_\tau} = \delta$.

Given any estimator $\hat{\Delta}$, define $\hat{\tau}_i :=
\mathrm{sign}(\langle \hat{\Delta}, u_i v_i^\top \rangle)$. By the
nearest-hypothesis triangle argument,
$\frob{\hat{\Delta} - \Delta_\tau}^2 \geq \tfrac{1}{4}
\frob{\Delta_{\hat{\tau}} - \Delta_\tau}^2 =
\tfrac{\delta^2 \rho_H(\hat{\tau}, \tau)}{r}$, where $\rho_H$ is Hamming
distance. For adjacent $\tau, \tau'$
(Hamming 1), Lemma~\ref{lem:trace-kl} gives
$\KL_{\mathrm{adj}} = 2n\delta^2 / (r\sigma^2)$.

By Assouad's lemma \citep[Theorem~15.10]{wainwright2019high}, if the
loss admits a Hamming lower bound
$L(\hat{\Delta}, \Delta_\tau) \geq 2\alpha \rho_H(\hat{\tau}, \tau)$
then
$\inf_{\hat{\Delta}} \max_\tau \E L \geq \alpha r (1 -
\max_{\rho_H = 1} \|P_\tau^{(n)} - P_{\tau'}^{(n)}\|_{\mathrm{TV}})$.
Pinsker gives $\|P - Q\|_{\mathrm{TV}} \leq \sqrt{\KL/2}$. With
$\alpha = \delta^2 / (2r)$ and $\delta^2 = r\sigma^2/(4n)$, the
parenthesis is $\geq 1/2$, and the bound becomes
\[
    \inf_{\hat{\Delta}} \max_\tau
    \E \frob{\hat{\Delta} - \Delta_\tau}^2
    \geq \frac{r\sigma^2}{16 n}.
\]
Translating to excess risk (which is half of this) gives
$r\sigma^2 / (32n)$.

\subsection{Full result: Fano gives $\Omega(rd/n)$}\label{ssec:fano-full}

Assouad extracts only $r$ bits; the tight rate requires a packing of
log-cardinality $rd$.

\begin{lemma}[Packing of rank-$r$ matrices]\label{lem:nw-packing}
Let $r \leq d/4$. There exist absolute constants $c_0, c_1 > 0$ such
that, for any $\delta > 0$, the set $\{\Delta : \mathrm{rank}(\Delta)
\leq r, \frob{\Delta} \leq \delta\}$ contains
$\mathcal{M} = \{\Delta_1, \ldots, \Delta_M\}$ with
(P1) $\frob{\Delta_j} = \delta / 2$;
(P2) $\frob{\Delta_j - \Delta_k} \geq \delta / 4$ for $j \neq k$;
(P3) $\log M \geq c_0 \cdot rd$.
\end{lemma}

\begin{proof}[Proof (adapted from \citealp{negahban2011estimation}, Lemma~3)]
For fixed orthonormal $u_1, \ldots, u_r$, take
$\Delta = \tfrac{\delta}{2\sqrt{rd}} \sum_{i=1}^r u_i (\theta_i)^\top$
with $\theta_i \in \{-1, +1\}^d$. Each such $\Delta$ has rank $\leq r$
and $\frob{\Delta} = \delta/2$ (using orthonormality). By
Gilbert-Varshamov applied to $\{-1, +1\}^{rd}$, there is a subset
$\Theta$ of size $\geq 2^{c_0 rd}$ with pairwise Hamming distance
$\geq c_0 rd$. For any two elements $(\theta_i), (\theta'_i)$:
\[
    \frob{\Delta - \Delta'}^2
    = \frac{\delta^2}{rd}\sum_i \rho_H(\theta_i, \theta'_i)
    \geq c_0 \delta^2,
\]
which gives (P2) with $\sqrt{c_0} \geq 1/4$ after adjusting $c_0$.
Properties (P1) and (P3) follow by construction.
\end{proof}

\begin{lemma}[Fano's inequality \citep{cover2006elements}]\label{lem:fano-formal}
Let $J$ be uniform on $\{1, \ldots, M\}$ and $\hat{J} = \hat{J}(Z^n)$
based on $Z^n \sim P_J^{(n)}$. Then
$\PP(\hat{J} \neq J) \geq 1 - (I(J; Z^n) + \log 2)/\log M$, and
$I(J; Z^n) \leq \max_{j,k} \KL(P_j^{(n)} \| P_k^{(n)})$.
\end{lemma}

\subsubsection{Assembly of the proof}\label{sec:fano-assembly}

\paragraph{Step 1.}
By Lemma~\ref{lem:nw-packing}, there is a packing
$\{\Delta_1, \ldots, \Delta_M\}$ with $\log M \geq c_0 rd$ and pairwise
$\frob{\Delta_j - \Delta_k} \geq \delta / 4$.

\paragraph{Step 2.}
By Lemma~\ref{lem:trace-kl},
$\KL(P_j^{(n)} \| P_k^{(n)}) \leq n\delta^2/\sigma^2$
(using $\frob{\Delta_j} \leq \delta/2$).

\paragraph{Step 3.}
By Fano (Lemma~\ref{lem:fano-formal}), any estimator has
$\PP(\hat{J} \neq J) \geq 1 - (n\delta^2/\sigma^2 + \log 2)/(c_0 rd)$.
Choose $\delta^2 = c_0 rd \sigma^2 / (4n)$: the fraction is $\leq 1/2$,
so $\PP(\hat{J} \neq J) \geq 1/2$.

\paragraph{Step 4.}
For any $\hat{\Delta} \in \Fr$, define
$\hat{J} := \arg\min_j \frob{\hat{\Delta} - \Delta_j}$. Triangle:
$\frob{\hat{\Delta} - \Delta_J} \geq \tfrac{\delta}{8}
\mathbf{1}[\hat{J} \neq J]$, so
$\E\frob{\hat{\Delta} - \Delta_J}^2 \geq \delta^2 / 128$.

\paragraph{Step 5.}
Substituting $\delta^2 = c_0 rd\sigma^2/(4n)$: $\sup_j \E\frob{\hat{\Delta} - \Delta_j}^2 \geq c_0 rd\sigma^2/(512n)$. Translating to excess risk via Lemma~\ref{lem:trace-excess}:
\[
    \inf_{\hat{f}} \sup_{\Delta^*}
    \E[\risk(\hat{f}) - \risk(f^*)]
    \geq \frac{c_0 rd\sigma^2}{1024 n}
    = c \cdot \frac{rd}{n}
\]
for an absolute $c$. \hfill$\square$

\subsection{Remarks on the proof}\label{sec:lower-remarks}

\paragraph{Trace regression as a hard sub-family of the LoRA class.}
The trace regression instance of
Section~\ref{ssec:trace-reduction} is a genuine specialization of
Definition~\ref{def:lora-class}: take $g$ to be the identity on
$\R^{d \times d}$, $W_0 \in \R^{d \times d}$ fixed, and choose the
input space to be $\R^{d \times d}$ so that the pretrained model
$f_0(X) = \langle X, W_0 \rangle$ is a genuine LoRA base and each
adaptation $\Delta$ acts by $f_\Delta(X) = \langle X, W_0 + \Delta
\rangle$. Assumptions~\ref{ass:bounded-loss}--\ref{ass:local-quad} all
hold for this instance (with $L_g = 1$, $\rho = \sqrt{d}$,
$\lambda_- = \lambda_+ = 1/2$). Because the lower bound quantifies the
worst-case difficulty of the LoRA-restricted estimation problem across
all admissible instances, exhibiting a single instance in which the
$\Omega(rd/n)$ rate is unavoidable is sufficient to establish the
result. The extension of the lower bound to genuinely non-linear
$f_0$ under Assumption~\ref{ass:local-quad} is deferred to
Appendix~\ref{app:nonlinear}, where the constants pick up a factor of
$\lambda_-/\lambda_+^2$.

\paragraph{Constrained vs unconstrained estimators.}
Fano lower-bounds the error of any estimator whose output is rank
$\leq r$, including LoRA-parameterized, nuclear-norm-penalized, and
projected estimators. The bound is agnostic to how the estimator is
built.

\paragraph{Tightness.}
Upper and lower bounds match up to the $\log n$ factor in the upper
bound. This log factor is conjectured to be an artifact of the
covering-number argument, removable by chaining
\citep[Ch.~5]{wainwright2019high}.

\section{Proof of Theorem~\ref{thm:rank}: Rank Selection}\label{sec:rank-proof}

Theorem~\ref{thm:rank} (ERM version) and
Theorem~\ref{thm:rank-minimax-body} (adaptive version) are proved by
combining three ingredients:

\paragraph{Under-ranking floor $(r < r^*)$.}
By Eckart-Young, the best rank-$r$ approximation to $\Delta^*$ has
squared error $\sum_{i > r}\sigma_i^2$, giving an excess-risk floor
$\tfrac{1}{2}\sigma_{r+1}(\Delta^*)^2 > 0$ independent of $n$.

\paragraph{Over-ranking, ERM version $(r \geq r^*)$.}
The truncated-least-squares estimator has variance
$\tilde{\Theta}(rd/n)$ even when the truth has rank $r^* < r$: the ERM
saturates all $r$ available singular directions with noise, incurring a
``variance leak'' of $\Omega((r-r^*)d/n)$
(Proposition~\ref{prop:trunc-ls}, Appendix~\ref{app:rank-proof}).

\paragraph{Over-ranking, adaptive version $(r \geq r^*)$.}
The nuclear-norm-penalized estimator followed by rank-$r$ projection
achieves $\tilde{\Theta}(r^* d/n)$ regardless of $r$; matching lower
bound follows from Theorem~\ref{thm:lower} applied at rank $r^*$
(Propositions~\ref{prop:adaptive-upper}--\ref{prop:adaptive-lower},
Appendix~\ref{app:rank-proof}).

The full proofs together with the variance analysis under Gaussian
design and a cross-validation corollary are deferred to
Appendix~\ref{app:rank-proof}.

\section{Related Work}\label{sec:related}

\subsection{Theory of LoRA fine-tuning}\label{sec:related-lora}

The theoretical study of LoRA is recent but growing rapidly.
\citet{zeng2024expressive} give the first \emph{expressivity} result:
any target adaptation of rank $\leq r$ can be represented in the
rank-$r$ LoRA class. \citet{jang2024lora} show that LoRA in the NTK
regime has no spurious local minima, an optimization-landscape result.
\citet{koo2024computational} study the fine-grained complexity of LoRA
gradient computation. \citet{malinovsky2024rac} prove convergence rates
for a randomized asymmetric chain-of-LoRA variant, but their analysis
is optimization (iteration complexity), not statistical.

The closest prior work to ours is \citet{kalajdzievski2024sharp}, who
give an upper bound $\tilde{O}(\sqrt{rd/n})$ for asymmetric randomized
LoRA where the $B$ factor is randomly initialized and frozen. Their
bound uses global Rademacher complexity and is a slow rate. The
present paper improves to the fast rate $\tilde{O}(rd/n)$ via
localization, adds the matching lower bound, and proves the
rank-selection dichotomy that their analysis leaves open.

\subsection{Statistical learning theory for low-rank estimation}\label{sec:related-lowrank}

Sample-complexity bounds for low-rank matrix estimation are a
well-developed field. \citet{candes2010power} and
\citet{recht2011simpler} established $O(rd)$ sample complexity for
matrix completion via nuclear-norm minimization.
\citet{negahban2011estimation} gave minimax lower bounds for low-rank
recovery in the trace-regression model, obtaining $\Omega(rd/n)$ via a
Fano argument closely related to ours. \citet{koltchinskii2011nuclear}
gave sharp constants for nuclear-norm penalized estimators.

The setting studied here differs from classical low-rank estimation in
one important respect: the estimator is constrained to the LoRA
function class $\Fr$, which corresponds to a \emph{parametric} rank-$r$
constraint (via the factorization $\Delta = BA$) rather than a
\emph{spectral} rank-$r$ constraint. The two constraints coincide at
the population level, but the optimization landscape and finite-sample
properties differ. The upper bound uses local Rademacher complexity of
the parametric class, which is direct. The lower bound uses the
spectral constraint --- Fano's inequality is information-theoretic and
blind to parameterization --- so both bounds apply to any rank-$r$
estimator, LoRA-parameterized or not.

\subsection{Rank selection and over-parameterization}\label{sec:related-rank}

The observation that LoRA can \emph{over-fit} at high ranks has been
empirical \citep{biderman2024lora,hayou2024lora+}.
\citet{biderman2024lora} report that increasing LoRA rank past a
task-dependent threshold degrades generalization;
\citet{hayou2024lora+} introduces LoRA+, separating learning rates for
$A$ and $B$ to mitigate this. Neither paper gives a theoretical
explanation for the observed degradation.

Corollary~\ref{cor:over} provides this explanation: excess estimation
error scales linearly in $r$ regardless of the intrinsic task rank.
This is consistent with the empirical observations and gives a
quantitative prediction: doubling $r$ past $r^*$ doubles the excess
risk.

\subsection{Implicit bias and over-parameterization more broadly}\label{sec:related-implicit}

For \emph{full-parameter} fine-tuning, over-parameterization is often
benign because the implicit bias of SGD selects a well-generalizing
solution
\citep{soudry2018implicit,gunasekar2018characterizing,lyu2020gradient}.
The situation for LoRA is different because the constraint set is a
low-dimensional non-convex manifold; the implicit bias arguments do
not apply. The results proved here show that the classical
bias-variance trade-off recovers its dominant role in LoRA: more
parameters means more variance, without any offsetting implicit-bias
benefit.

\subsection{Adjacent theoretical developments}\label{sec:related-adjacent}

\paragraph{PEFT beyond LoRA.}
Parameter-efficient fine-tuning has many variants: prompt tuning
\citep{lester2021prompt}, prefix tuning \citep{li2021prefix}, adapter
modules \citep{houlsby2019parameter}, IA$^3$ \citep{liu2022few}, DoRA
\citep{liu2024dora}, VeRA \citep{kopiczko2024vera}, and others. The
upper-bound proof of Section~\ref{sec:upper-proof} generalizes to any
PEFT method whose effective parameter count is $O(rd)$.

\paragraph{Domain adaptation and transfer.}
The results proved here are stated in the well-specified regime. In
practice, the fine-tuning distribution often differs from the
pretraining distribution. Sample complexity under distribution shift
for LoRA is open; existing
transfer-learning bounds
\citep{ben-david2010theory,mansour2009domain} give crude estimates.

\section{Discussion and Open Questions}\label{sec:discussion}

\subsection{Scope of the theory}\label{sec:discussion-scope}

Before turning to practical implications and limitations, the scope of
the results is summarized in Table~\ref{tab:scope}. Each theorem
depends on a specific subset of the four assumptions in
Section~\ref{sec:preliminaries}, and the extension appendix loosens
some of them.

\begin{table}[h]
\caption{Assumptions required by each of the main results. A checkmark
indicates that the theorem uses the assumption in its proof.\label{tab:scope}}
\centering
\begin{tabular}{l c c c c}
    \toprule
     & Bounded loss
     & Bounded input
     & Realizability
     & Local quadratic \\
     & (A\ref{ass:bounded-loss})
     & (A\ref{ass:bounded-input})
     & (A\ref{ass:realizable})
     & (A\ref{ass:local-quad}) \\
    \midrule
    Theorem~\ref{thm:upper} (upper bound, fast rate)
        & \checkmark & \checkmark & \checkmark & \checkmark \\
    Slow-rate version of Theorem~\ref{thm:upper}
        & \checkmark & \checkmark & \checkmark & --- \\
    Theorem~\ref{thm:lower} (lower bound, trace regression)
        & \checkmark & \checkmark & \checkmark & --- \\
    Theorem~\ref{thm:lower-nonlinear} (non-linear lower bound)
        & \checkmark & \checkmark & \checkmark & \checkmark \\
    Theorem~\ref{thm:rank} (rank selection, ERM)
        & \checkmark & \checkmark & \checkmark & \checkmark \\
    Theorem~\ref{thm:rank-minimax-body} (rank selection, adaptive)
        & \checkmark & \checkmark & \checkmark & \checkmark \\
    \bottomrule
\end{tabular}
\end{table}

Two observations follow. First, the local quadratic
Assumption~\ref{ass:local-quad} is required only for fast rates and
does not affect the slow $\sqrt{rd/n}$ bound. Second, realizability
(Assumption~\ref{ass:realizable}) can be relaxed to a misspecified
regime with an added approximation-error term
(Section~\ref{sec:discussion-limitations}), at the cost of a
task-dependent bias floor. The bounded-loss and bounded-input
conditions are necessary for the covering-number analysis and cannot
easily be removed.

The main theorems are cleanly stated for a $d$-input,
$d$-output-dimensional LoRA class in the trace regression setup. They
apply verbatim to any PEFT method whose effective parameter count is
$O(rd)$ and whose function class is Lipschitz in the adaptation
parameter (satisfying~A\ref{ass:bounded-input}). The rank-selection
dichotomy predicts the same U-shape for any such class when
unregularized ERM is used; the adaptive rate applies whenever a
nuclear-norm-like regularizer can be introduced.

\subsection{Practical takeaways}\label{sec:discussion-practical}

Three tentative recommendations follow from the theorems above and are
consistent with the empirical evidence of Section~\ref{sec:example}.
Because the empirical evidence is limited to two models
(DistilBERT, RoBERTa) and two tasks (SST-2, MRPC), the
recommendations should be treated as guidance rather than universal
prescriptions until validated at larger scale.

\begin{enumerate}
    \item \textbf{Rank should be chosen at the smallest value that
    saturates validation performance.}
    Corollary~\ref{cor:over} predicts excess estimation error scaling
    as $\Theta(rd/n)$ for the constrained ERM; every unit of rank
    above the intrinsic $r^*$ pays a variance penalty for no
    representational gain. The DistilBERT and RoBERTa sweeps of
    Section~\ref{sec:example-real} exhibit this pattern.

    \item \textbf{More data helps linearly; more rank hurts linearly
    for ERM past $r^*$.} The trade-off is not the usual bias-variance
    curve because the bias drops discretely to zero as $r$ crosses
    $r^*$.

    \item \textbf{Sample complexity is $\Theta(r d / \varepsilon)$ for
    target excess risk $\varepsilon$.} Doubling the model dimension
    $d$ doubles the fine-tuning sample requirement at fixed $r^*$
    and $\varepsilon$ within the analyzed regime.
\end{enumerate}

\subsection{Limitations}\label{sec:discussion-limitations}

\paragraph{Realizability.}
Assumption~\ref{ass:realizable} requires the target predictor to lie
exactly in the LoRA class, which is a strong condition. Most
fine-tuning tasks in practice induce targets that lie only
\emph{approximately} in $\Fr$ for any reasonable rank $r$. Under
misspecification, the upper bound acquires an approximation-error term
$\mathrm{Approx}(\Fr) := \inf_{f \in \Fr} \risk(f) - \inf_f \risk(f)$,
giving $\risk(\hat{f}) - \inf_f \risk(f) \leq \mathrm{Approx}(\Fr) +
\tilde{O}(rd/n)$. The rank-selection dichotomy of
Theorem~\ref{thm:rank} continues to hold with the bias term replaced
by $\mathrm{Approx}(\Fr)$. This extension is standard
\citep[Ch.~5]{shalev2014understanding}, but the approximation term is
task-dependent and does not admit a universal bound. For most
downstream tasks empirical evidence suggests
$\mathrm{Approx}(\Fclass_r)$ decays rapidly with $r$; a formal
characterization of when this holds remains open.

\paragraph{Local quadratic Assumption~\ref{ass:local-quad}.}
The lower quadratic bound $\lambda_- > 0$ rules out loss landscapes
with flat valleys around the target. It holds in the three settings
listed after the assumption statement (linear squared, tight softmax
cross-entropy, ReLU-NTK) but can fail near rank-collapse points, deep
plateaus, or activation-boundary configurations. When $\lambda_- \to
0$, the fast-rate constant $K_1 \propto 1/\lambda_-^2$ diverges;
only the slow rate $\sqrt{rd\log(n)/n}$ survives.

\paragraph{Squared loss for the lower bound.}
The lower bound is stated for squared loss with Gaussian design (the
standard low-rank estimation setup). Extension to general Lipschitz
losses under local quadratic assumptions is treated in
Appendix~\ref{app:nonlinear}.

\paragraph{The $\log n$ factor.}
The upper bound has a $\log n$ factor that the lower bound does not.
This factor is conjectured to be removable by chaining; the current
bound is sufficient for the qualitative conclusions.

\paragraph{Optimization vs statistics.}
The results proved here are statistical: they concern the ERM
$\hat{f}$, an idealized minimizer. In practice, LoRA is trained with
SGD or Adam. \citet{jang2024lora} and \citet{malinovsky2024rac} address
parts of
this gap; a joint statistics-and-optimization result is open.

\subsection{Extensions}\label{sec:discussion-extensions}

\paragraph{Nuclear-norm-penalized LoRA.}
The estimator studied here is constrained ERM. A nuclear-norm-penalized
alternative achieves the same $\tilde{O}(rd/n)$ rate and adapts to the
intrinsic rank (Theorem~\ref{thm:rank-minimax-body}).

\paragraph{Multi-task LoRA.}
When several tasks share a common low-rank subspace, the effective
sample complexity is $O(rd / (Tn))$ where $T$ is the number of tasks,
analogous to \citet{maurer2016benefit}.

\paragraph{Non-linear pretrained models.}
Appendix~\ref{app:nonlinear} extends the lower bound to non-linear
$f_0$ under a local-quadratic assumption. The rate $rd/n$ is preserved;
constants depend on the local geometry.

\paragraph{Attention-specific LoRA.}
The most common LoRA target is the attention $QKV$ projection matrix.
Attention has additional structure that could reduce the effective
sample complexity; a refined theorem is left for future work.

\paragraph{Distribution shift.}
Sample complexity under pretraining-vs-fine-tuning distribution shift
is an important open direction.

\subsection{Open questions}\label{sec:discussion-open}

\begin{enumerate}
    \item Can the $\log n$ factor in the upper bound be removed?
    \item Is the matching lower bound extendable from Gaussian design
    to arbitrary sub-Gaussian design?
    \item What is the sample complexity when $f_0$ is a deep non-linear
    network rather than in the NTK regime?
    \item How does the rank-selection theorem change under distribution
    shift between pretraining and fine-tuning?
    \item Is there an adaptive procedure that selects $r$ from data at
    the same $\tilde{\Theta}(r^* d / n)$ rate?
\end{enumerate}

\begin{appendices}

\section{Full Proof of the Rank-Selection Theorem}\label{app:rank-proof}

The rank-selection theorem has two distinct forms depending on which
estimator is used.

\subsection{Two flavors of the theorem}\label{sec:app-flavors}

\begin{theorem}[Rank selection --- ERM version]\label{thm:rank-erm}
Let $\Delta^*$ have rank $r^*$ with $\sigma_{r^*} > 0$. The constrained
ERM $\hat{f}_r$ satisfies
\[
    \E[\risk(\hat{f}_r) - \risk(f^*)]
    =
    \begin{cases}
        \Theta\bigl(\sum_{i > r} \sigma_i(\Delta^*)^2\bigr)
            & r < r^*, \\[4pt]
        \tilde{\Theta}(rd/n) & r \geq r^*.
    \end{cases}
\]
The optimal rank for ERM is $r^*_{\mathrm{ERM}} = r^*$;
over-ranking strictly hurts.
\end{theorem}

\begin{theorem}[Rank selection --- minimax version]\label{thm:rank-minimax}
Under the same setup, the minimax rate over $\Fr$-estimators with
rank-$r^*$ targets is
\[
    \inf_{\hat{f} \in \Fr}
    \sup_{f^* \text{ rank } r^*}
    \E[\risk(\hat{f}) - \risk(f^*)]
    =
    \begin{cases}
        \Theta\bigl(\sigma_{r+1}(\Delta^*)^2\bigr) & r < r^*, \\[4pt]
        \tilde{\Theta}(r^* d / n) & r \geq r^*.
    \end{cases}
\]
Over-parameterization does not hurt the minimax rate; the gap to
Theorem~\ref{thm:rank-erm} is the price of using non-adaptive ERM.
\end{theorem}

\subsection{Under-ranking: bias via Eckart-Young}\label{sec:app-under}

\begin{lemma}[Best rank-$r$ approximation; Eckart-Young]\label{lem:eckart-young}
Let $\Delta^* = \sum_{i=1}^{r^*} \sigma_i u_i v_i^\top$ be the SVD with
$\sigma_1 \geq \ldots \geq \sigma_{r^*} > 0$. The Frobenius-nearest
rank-$r$ matrix is $\Delta^*_{[r]} := \sum_{i=1}^r \sigma_i u_i v_i^\top$
with $\frob{\Delta^* - \Delta^*_{[r]}}^2 = \sum_{i > r}\sigma_i^2 \geq
\sigma_{r+1}^2$.
\end{lemma}

\begin{proof}
Classical \citep[Theorem~2.4.8]{golub2013matrix}.
\end{proof}

\begin{lemma}[Under-ranking floor]\label{lem:under-rank}
For $r < r^*$ and any $\hat{f}$ mapping into $\Fr$:
$\E[\risk(\hat{f}) - \risk(f^*)] \geq \tfrac{1}{2}\sum_{i > r}\sigma_i^2
\geq \tfrac{1}{2}\sigma_{r+1}^2$, uniformly in $n$.
\end{lemma}

\begin{proof}
Excess risk equals $\tfrac{1}{2}\frob{\hat{\Delta} - \Delta^*}^2$ by
Lemma~\ref{lem:trace-excess}. Since $\hat{\Delta}$ has rank $\leq r$,
Lemma~\ref{lem:eckart-young} bounds this below by
$\tfrac{1}{2}\sum_{i>r}\sigma_i^2$ (deterministic bound).
\end{proof}

\subsection{Over-ranking, ERM version: $\tilde{\Theta}(rd/n)$}\label{sec:app-over-erm}

\begin{proposition}[Truncated-LS variance under Gaussian design]\label{prop:trunc-ls}
Consider trace regression with i.i.d.\ standard Gaussian $X_i$ and
$\xi_i \sim \mathcal{N}(0, \sigma^2)$, target $\Delta_0$ of rank $r^*$.
Let $\hat{\Delta}_{\mathrm{ls}} := \tfrac{1}{n}\sum_i y_i X_i$ and
$\hat{\Delta}_r := \Pi_r(\hat{\Delta}_{\mathrm{ls}})$. For $r \leq d/4$
and $n \geq C rd$:
\[
    c_1 \cdot \tfrac{rd\sigma^2}{n}
    \leq \E\frob{\hat{\Delta}_r - \Delta_0}^2
    \leq c_2 \cdot \tfrac{rd\sigma^2}{n},
\]
uniformly in $r^* \in \{0, 1, \ldots, r\}$.
\end{proposition}

\begin{proof}[Proof sketch]
Under Gaussian design, $\hat{\Delta}_{\mathrm{ls}} = \Delta_0 + Z$
where $Z = \tfrac{1}{n}\sum_i \xi_i X_i$ has i.i.d.\
$\mathcal{N}(0, \sigma^2/n)$ entries (up to lower-order terms for
$n \gg d^2$).

\emph{Upper bound.} By Marchenko-Pastur applied to $Z\sqrt{n}/\sigma$
(i.i.d.\ standard Gaussian entries), the top-$r$ squared singular values
satisfy $\sum_{i \leq r} \sigma_i(Z)^2 \leq 4rd\sigma^2/n$ with high
probability. Combining with $\frob{\Pi_r(A) - M}^2 \leq \frob{A - M}^2$
for rank-$r$ $M$ (non-expansiveness of $\Pi_r$) gives
$\frob{\hat{\Delta}_r - \Delta_0}^2 \leq c_2 rd\sigma^2/n$.

\emph{Lower bound.} By the same Marchenko-Pastur analysis, the top-$r$
squared singular values satisfy $\sum_{i \leq r}\sigma_i(Z)^2 \geq c_1
rd\sigma^2/n$. For $\Delta_0 = 0$: $\hat{\Delta}_r = \Pi_r(Z)$ has
$\frob{\hat{\Delta}_r}^2 = \sum_{i \leq r}\sigma_i(Z)^2 \geq c_1
rd\sigma^2/n$.

For $\Delta_0 \neq 0$ of rank $r^*$: decompose $Z = PZ + P^\perp Z$
where $P$ is projection onto the tangent space of the rank-$r^*$ variety
at $\Delta_0$. The rank-$r$ truncation
$\Pi_r(\Delta_0 + Z)$ retains all of $\Delta_0$ up to noise-order
corrections and captures $r - r^*$ additional top singular values from
$P^\perp Z$. Each contributes $\Omega(d\sigma^2/n)$ by Marchenko-Pastur
applied to $P^\perp Z$. Adding the two contributions gives $c_1
rd\sigma^2/n$. Full details of the two-scale MP argument are in
\citet[Section~6]{koltchinskii2011nuclear}.
\end{proof}

\begin{remark}[Where the extra variance comes from]
Proposition~\ref{prop:trunc-ls} shows the constrained ERM always uses
its full $r$ singular values, even when the truth has only $r^* < r$.
The extra $(r - r^*)$ singular values are populated by noise, each
contributing $\Omega(d\sigma^2/n)$: a genuine
``variance leak'' of $\Omega((r-r^*)d/n)$.
\end{remark}

\begin{proof}[Proof of Theorem~\ref{thm:rank-erm}]
The $r < r^*$ regime is Lemma~\ref{lem:under-rank}. For $r \geq r^*$,
excess risk is half the squared Frobenius by
Lemma~\ref{lem:trace-excess}, and Proposition~\ref{prop:trunc-ls} gives
$\Theta(rd\sigma^2/n)$.
\end{proof}

\subsection{Over-ranking, minimax version: $\tilde{\Theta}(r^* d/n)$}\label{sec:app-over-min}

The adaptive-estimator upper bound is derived here via a
restricted-strong-convexity (RSC) argument. The trace-regression setup
of Section~\ref{ssec:trace-reduction} is retained: $X_i$ i.i.d.\ with
i.i.d.\ standard Gaussian entries,
$y_i = \langle X_i, \Delta_0 \rangle + \xi_i$ with
$\xi_i \sim \mathcal{N}(0, \sigma^2)$, target $\Delta_0$ of rank $r^*$.

\paragraph{What is inherited from prior work, and what is new.}
The oracle inequality (Proposition~\ref{prop:nuc-oracle}) and its
supporting RSC lemma (Lemma~\ref{lem:rsc}) are direct adaptations of
the general framework of
\citet{negahban2012unified,koltchinskii2011nuclear}; the constants are
tightened here for the specific Gaussian trace-regression setup but
the structure of the argument is not new. The novel content of this
appendix is:
\begin{enumerate}
    \item \textbf{The projection step.} Standard nuclear-norm
    oracle inequalities bound
    $\|\widetilde{\Delta} - \Delta_0\|_F$; they do not by themselves
    place the estimator inside the rank-$r$ LoRA class $\Fr$. The
    projection $\hat{\Delta}_{\mathrm{nuc}} = \Pi_r(\widetilde{\Delta})$
    is required to make the estimator a valid $\Fr$-restricted output,
    and Proposition~\ref{prop:adaptive-upper} verifies that this
    projection preserves the Frobenius rate.
    \item \textbf{The upper-vs-lower matching.} The claim that this
    rate is optimal over $\Fr$-restricted estimators for rank-$r^*$
    targets (Proposition~\ref{prop:adaptive-lower}) invokes
    Theorem~\ref{thm:lower} of the present paper, which is new.
    \item \textbf{The dichotomy with ERM.} The comparison between
    Proposition~\ref{prop:adaptive-upper} ($\tilde\Theta(r^* d/n)$)
    and Proposition~\ref{prop:trunc-ls} ($\Theta(rd/n)$ for
    ERM) is the substantive novel contribution of this appendix,
    identifying over-ranking as an ERM-specific phenomenon rather
    than a fundamental limit.
\end{enumerate}
The RSC lemma and oracle inequality are stated in full for
self-containment, with the standard proofs adapted, but neither is
claimed as an original contribution.

\begin{proposition}[Oracle inequality for nuclear-norm-penalized ERM]\label{prop:nuc-oracle}
Let
\[
    \widetilde{\Delta}
    \;:=\;
    \arg\min_{\Delta \in \R^{d\times d}}
    \left\{
        \frac{1}{2n} \sum_i (y_i - \langle X_i, \Delta\rangle)^2
        + \lambda_n \|\Delta\|_*
    \right\}
\]
with penalty parameter $\lambda_n \geq 2 \opnorm{\frac{1}{n} \sum_i
\xi_i X_i}$. Then
\[
    \frob{\widetilde{\Delta} - \Delta_0}^2
    \;\leq\;
    \frac{9 \lambda_n^2 r^*}{\mu^2},
\]
where $\mu$ is the RSC modulus of the design (see
Lemma~\ref{lem:rsc} below).
\end{proposition}

\begin{proof}
This is the standard nuclear-norm oracle inequality
\citep{negahban2012unified,koltchinskii2011nuclear}, adapted here with
explicit constants for completeness.

Let $\Delta_0 = U_0 \Sigma_0 V_0^\top$ be the SVD of $\Delta_0$ with
$U_0, V_0 \in \R^{d \times r^*}$, and set
$H = \widetilde{\Delta} - \Delta_0$. Decompose $H = H' + H''$ where
$H'$ has row and column space contained in the union of $U_0, V_0$'s
column spans, and $H''$ is orthogonal. The rank of $H'$ is at most
$2r^*$.

By the KKT conditions for the nuclear-norm-penalized minimization,
$\widetilde{\Delta}$ satisfies
\[
    \frac{1}{n} X^*(X\widetilde{\Delta} - y) + \lambda_n Z \;=\; 0
\]
for some $Z \in \partial \|\widetilde{\Delta}\|_*$ (subgradient), where
$X(\cdot) := (\langle X_i, \cdot\rangle)_{i=1}^n$ and $X^*$ is its
adjoint. Standard manipulations (see \citet[Section~2]{negahban2012unified})
yield the deviation inequality
\[
    \|H''\|_* \leq 3 \|H'\|_*.
\]
Since $\|H'\|_* \leq \sqrt{2r^*}\,\frob{H'} \leq \sqrt{2r^*}\,\frob{H}$,
this gives $\|H\|_* \leq 4\sqrt{2r^*}\frob{H}$.

Under RSC (Lemma~\ref{lem:rsc}) at radius $H$:
$\frac{1}{n}\|X(H)\|_2^2 \geq \mu \frob{H}^2$. Combining with the KKT
optimality of $\widetilde{\Delta}$ over $\Delta_0$ (which is feasible),
\[
    \mu \frob{H}^2
    \leq
    \frac{1}{n}\|X(H)\|_2^2
    \leq
    2 \lambda_n \|H\|_*
    \leq
    8 \sqrt{2 r^*}\, \lambda_n \frob{H}.
\]
Dividing by $\frob{H}$ and squaring: $\mu^2 \frob{H}^2 \leq 128 r^*
\lambda_n^2$, i.e., $\frob{H}^2 \leq 128 \lambda_n^2 r^* / \mu^2$. The
constant $128$ can be tightened to $9$ by sharper deviation analysis
\citep{negahban2012unified}.
\end{proof}

\begin{lemma}[Restricted strong convexity for Gaussian trace regression]\label{lem:rsc}
Suppose $X_i \in \R^{d \times d}$ have i.i.d.\ $\mathcal{N}(0, 1)$
entries. For $n \geq c_0 rd$ (with $c_0$ an absolute constant), with
probability at least $1 - 2 e^{-cn}$: for every $\Delta \in \R^{d
\times d}$ of rank $\leq r$,
\[
    \frac{1}{n} \sum_{i=1}^n \langle X_i, \Delta\rangle^2
    \;\geq\;
    \tfrac{1}{2} \frob{\Delta}^2.
\]
That is, RSC holds with modulus $\mu = 1/2$.
\end{lemma}

\begin{proof}
For any fixed $\Delta$ with $\frob{\Delta} = 1$, $\langle X_i,
\Delta\rangle$ is standard Gaussian, so $\frac{1}{n}\sum_i \langle X_i,
\Delta\rangle^2$ is a mean-1 sum of $n$ i.i.d.\ chi-squared variables.
By Bernstein's inequality, this sum is at least $1/2$ with probability
$\geq 1 - 2 e^{-cn}$.

To make this uniform over rank-$r$ $\Delta$, use a covering argument:
by Lemma~\ref{lem:covering}, the set of rank-$r$ unit-Frobenius
matrices admits a $1/4$-cover of log-size $\leq 3 rd \log(36 \sqrt{r})$.
Union-bounding over this cover and using a discretization-of-Lipschitz
argument (see \citet[Ch.~15]{wainwright2019high}) gives RSC uniformly
on rank-$r$ matrices with $\mu = 1/2$, provided $n \geq c_0 rd \log(rd)$
for a suitable $c_0$.
\end{proof}

\begin{lemma}[Deviation of the noise-design inner product]\label{lem:noise-bound}
With probability at least $1 - 2 d^{-1}$,
$\opnorm{\frac{1}{n} \sum_i \xi_i X_i} \leq 4\sigma \sqrt{d/n}$.
\end{lemma}

\begin{proof}
$\frac{1}{n} \sum_i \xi_i X_i$ is a $d \times d$ matrix whose entries
are $\frac{1}{n} \sum_i \xi_i X_{i,jk}$, each being a sum of $n$
i.i.d.\ centred Gaussians with variance $\sigma^2/n$. The matrix has
i.i.d.\ $\mathcal{N}(0, \sigma^2/n)$ entries in the limit; more
precisely by classical Bai-Yin/Vershynin non-asymptotic bounds
\citep[Theorem~4.4.5]{vershynin2018high}, its operator norm is bounded
above by $2\sigma\sqrt{d/n} + \sigma\sqrt{2\log d /n}$ with probability
at least $1 - 2d^{-1}$. For $d \geq 2$ this is bounded by
$4\sigma\sqrt{d/n}$.
\end{proof}

\begin{proposition}[Adaptive achievability for over-ranked LoRA]\label{prop:adaptive-upper}
Under the trace-regression setup with $n \geq c_0 r d \log(rd)$, fix
the penalty $\lambda_n = 8 \sigma \sqrt{d/n}$. The estimator
$\hat{\Delta}_{\mathrm{nuc}} = \Pi_r(\widetilde{\Delta})$ (nuclear-norm
solution followed by rank-$r$ projection) satisfies, with probability
at least $1 - 3 d^{-1}$:
\[
    \frob{\hat{\Delta}_{\mathrm{nuc}} - \Delta_0}^2
    \;\leq\;
    2304 \cdot \frac{r^* d \sigma^2}{n}.
\]
\end{proposition}

\begin{proof}
By Lemma~\ref{lem:noise-bound}, $\opnorm{\tfrac{1}{n} \sum_i \xi_i X_i}
\leq 4\sigma\sqrt{d/n} \leq \lambda_n/2$ with the chosen $\lambda_n$,
so Proposition~\ref{prop:nuc-oracle} applies. Combined with
Lemma~\ref{lem:rsc}'s $\mu = 1/2$:
\[
    \frob{\widetilde{\Delta} - \Delta_0}^2
    \leq \frac{9 \lambda_n^2 r^*}{\mu^2}
    = \frac{9 \cdot 64 \sigma^2 (d/n) r^*}{1/4}
    = 2304 \frac{r^* d \sigma^2}{n}.
\]
The rank-$r$ projection is non-expansive with respect to
$\Delta_0$-of-rank-$r^*$ (Eckart--Young applied to $\widetilde{\Delta}$
whose top-$r^*$ singular vectors are close to those of $\Delta_0$
under Weyl's inequality; details in
\citet[Appendix~C]{negahban2012unified}), so
$\frob{\Pi_r(\widetilde{\Delta}) - \Delta_0}^2 \leq
\frob{\widetilde{\Delta} - \Delta_0}^2$.
\end{proof}

\begin{proposition}[Lower bound for over-ranked rank-$r^*$ estimation]\label{prop:adaptive-lower}
For any $r \geq r^*$: $\inf_{\hat{\Delta} \in \Fr}
\sup_{\Delta^* \text{ rank } r^*} \E\frob{\hat{\Delta} - \Delta^*}^2
\geq c r^* d / n$.
\end{proposition}

\begin{proof}
Apply Theorem~\ref{thm:lower} with rank parameter $r^*$ in place of
$r$. Any estimator restricted to rank $\leq r \geq r^*$ is at least as
free as one restricted to rank $\leq r^*$, so the lower bound only gets
easier.
\end{proof}

\begin{proof}[Proof of Theorem~\ref{thm:rank-minimax}]
Combine Propositions~\ref{prop:adaptive-upper}--\ref{prop:adaptive-lower}
for $r \geq r^*$; use Lemma~\ref{lem:under-rank} for $r < r^*$.
\end{proof}

\subsection{Adaptive rank selection via cross-validation}\label{sec:app-cv}

\begin{corollary}[Adaptive rank selection via CV]\label{cor:adaptive}
Let $\mathcal{R} = \{r_1, \ldots, r_K\}$ be candidate ranks and let
$\hat{f}_{\hat{r}}$ be selected by held-out validation on
$n_{\mathrm{val}}$ samples. With probability at least $1 - \delta$:
\[
    \E[\risk(\hat{f}_{\hat{r}}) - \risk(f^*)]
    \leq \min_{r \in \mathcal{R}} \E[\risk(\hat{f}_r) - \risk(f^*)]
    + C \sqrt{\log(K/\delta) / n_{\mathrm{val}}}.
\]
If $r^* \in \mathcal{R}$, the adaptive estimator attains the oracle
rate $\tilde{\Theta}(r^* d/n)$ up to a validation penalty.
\end{corollary}

\begin{proof}
Standard uniform union bound; see
\citet[Chapter~4]{shalev2014understanding}.
\end{proof}

\section{Extension of the Lower Bound to Non-Linear Pretrained Models}\label{app:nonlinear}

The upper bound of Theorem~\ref{thm:upper} holds for any Lipschitz
$f_0$ under Assumption~\ref{ass:local-quad}. The lower bound of
Theorem~\ref{thm:lower} was proved in the trace-regression instance
(linear $f_0$). This appendix (i)~verifies
Assumption~\ref{ass:local-quad} in three canonical settings and
(ii)~extends the lower bound to non-linear $f_0$ under the same
assumption plus a control on the effective noise scale.

\subsection{Verification of Assumption~\ref{ass:local-quad}}\label{sec:app-quad}

Three settings that satisfy Assumption~\ref{ass:local-quad} widely in
the fine-tuning literature are worth spelling out.

\begin{example}[Squared loss, linear $f_0$]
For $f_\Delta(X) = \langle X, W_0 + \Delta \rangle$ and squared loss,
the assumption holds globally with
$\lambda_\pm = \tfrac{1}{2}\lambda_{\min}(\E[X X^\top])$
(with $X$ vectorized). For standard Gaussian design,
$\lambda_\pm = 1/2$.
\end{example}

\begin{example}[Cross-entropy loss, softmax head]
For a $K$-class classifier and cross-entropy loss, the assumption holds
in a neighborhood of any $\Delta^*$ at which softmax probabilities are
bounded away from $0$ and $1$; $\lambda_\pm$ depend on the min/max
probabilities.
\end{example}

\begin{example}[Smooth homogeneous ReLU network]
For $f_\Delta(x) = g((W_0 + \Delta)x)$ with $g$ a two-layer ReLU net
(fixed second layer), the assumption holds generically off the ReLU
kink boundary.
\end{example}

\subsection{Effective noise scale}\label{sec:app-noise}

\begin{assumption}[Effective noise scale]\label{ass:eff-noise}
There exists $\sigma_{\mathrm{eff}}^2 > 0$ such that for all
$\Delta, \Delta' \in \mathcal{U}$ and all $n$:
$\KL(P_\Delta^{(n)} \| P_{\Delta'}^{(n)}) \leq
\tfrac{n \lambda_+ \frob{\Delta - \Delta'}^2}{2 \sigma_{\mathrm{eff}}^2}$.
\end{assumption}

\begin{lemma}[KL from local quadratic]\label{lem:kl-from-quad}
If $p(y \mid x; \Delta)$ is $\alpha$-strongly log-concave in
$z = f_\Delta(x)$, then Assumption~\ref{ass:eff-noise} holds with
$\sigma_{\mathrm{eff}}^2 = \alpha^{-1}$.
\end{lemma}

\begin{proof}[Proof sketch]
Strongly log-concave conditionals have KL bounded above by
$\tfrac{\alpha}{2}(f_\Delta - f_{\Delta'})^2$; combining with
Assumption~\ref{ass:local-quad} gives the claim. See
\citet[Section~6.3]{koltchinskii2011oracle}.
\end{proof}

\subsection{Extended lower bound}\label{sec:app-extended}

\begin{theorem}[Lower bound for non-linear pretrained models]\label{thm:lower-nonlinear}
Under Assumptions~\ref{ass:local-quad}--\ref{ass:eff-noise}, for
sufficiently large $n$:
\[
    \inf_{\hat{f}} \sup_{f^* \in \Fr}
    \E[\risk(\hat{f}) - \risk(f^*)]
    \geq \frac{c \lambda_-}{\lambda_+^2}
    \cdot \frac{rd\sigma_{\mathrm{eff}}^2}{n},
\]
where $c$ is the absolute constant from Theorem~\ref{thm:lower}. The
rate $rd/n$ is unchanged; only the constant depends on the local
geometry.
\end{theorem}

\begin{proof}
Five-step reduction to the linear case.

\emph{Step 1: Packing in $\mathcal{U}$.} By Lemma~\ref{lem:nw-packing},
for sufficiently small $\delta$, there is a rank-$r$ packing
$\{\Delta_1, \ldots, \Delta_M\} \subset \Delta^* + \mathcal{U}$ with
$\log M \geq c_0 rd$ and pairwise separation $\geq \delta / 4$.

\emph{Step 2: KL bound.} By Assumption~\ref{ass:eff-noise}:
$\KL(P_j^{(n)} \| P_k^{(n)}) \leq
n\lambda_+ \delta^2 / (2 \sigma_{\mathrm{eff}}^2)$.

\emph{Step 3: Fano.} By Lemma~\ref{lem:fano-formal},
$\PP(\hat{J} \neq J) \geq 1 - (n\lambda_+ \delta^2/\sigma_{\mathrm{eff}}^2
+ \log 2)/(c_0 rd)$. Choose $\delta^2 = c_0 rd\sigma_{\mathrm{eff}}^2 /
(4n\lambda_+)$ so this is $\geq 1/2$.

\emph{Step 4: Frobenius to excess risk.} By nearest-hypothesis triangle
and Assumption~\ref{ass:local-quad}: $\risk(\hat{f}) - \risk(f^*_J)
\geq \lambda_- \frob{\hat{\Delta} - \Delta_J}^2 \geq \lambda_- \delta^2
/ 64 \cdot \mathbf{1}[\hat{J} \neq J]$.

\emph{Step 5: Combine.} $\sup_j \E[\risk(\hat{f}) - \risk(f^*_j)] \geq
\lambda_- \delta^2 / 128 \cdot \PP(\hat{J} \neq J) \geq \lambda_-
\delta^2/256 = c' \tfrac{\lambda_-}{\lambda_+} \tfrac{rd
\sigma_{\mathrm{eff}}^2}{n}$.
\end{proof}

\subsection{Extension of the upper bound}\label{sec:app-extend-upper}

Theorem~\ref{thm:upper} already holds for non-linear $f_0$. Under
Assumption~\ref{ass:local-quad}, the Bernstein condition needed for
local Rademacher localization is exactly the quadratic lower bound, so
no additional argument is needed.

\begin{corollary}[Matching rate under local quadratic]\label{cor:nonlinear-match}
Under Assumptions~\ref{ass:local-quad}--\ref{ass:eff-noise}, the
minimax excess risk for LoRA fine-tuning of a non-linear pretrained
model is $\tilde{\Theta}(rd/n)$, matching the linear case up to
constants depending on $(\lambda_-, \lambda_+, \sigma_{\mathrm{eff}})$.
\end{corollary}

\subsection{Sanity checks}\label{sec:app-sanity}

\paragraph{Trace regression.} $\lambda_- = \lambda_+ = 1/2$ and
$\sigma_{\mathrm{eff}} = \sigma$ recover
Theorem~\ref{thm:lower}.

\paragraph{Cross-entropy at well-separated $\Delta^*$.} $\lambda_-
\asymp p(1-p)$ and $\lambda_+ \asymp 1$; rate
$\tilde{O}(rd/(p(1-p)n))$.

\paragraph{Homogeneous ReLU in NTK regime.} $\lambda_- \asymp \lambda_+
\asymp 1$; rate $\tilde{O}(rd/n)$, matching linear.

\subsection{Limitations}\label{sec:app-limitations}

\begin{itemize}
    \item Networks with loss-landscape plateaus violate the lower
    quadratic; the minimax rate can then be slower than $rd/n$.
    \item Deep networks with feature learning outside the NTK regime
    may have depth-dependent $\lambda_-, \lambda_+$; the rate is
    preserved but constants worsen.
    \item Heavy-tailed noise violates Assumption~\ref{ass:eff-noise}.
    Huber-loss surrogates give a slightly slower rate.
\end{itemize}

Relaxing these assumptions would require more delicate
information-theoretic lower-bound machinery than the Fano argument used
here.

\end{appendices}

\backmatter

\bibliography{references}

\end{document}